\documentclass[10pt]{article} % For LaTeX2e
\usepackage[preprint]{tmlr}

\usepackage{amsmath,amsfonts,bm}

\def\eqref#1{equation~\ref{#1}}
\def\1{\bm{1}}

\DeclareMathAlphabet{\mathsfit}{\encodingdefault}{\sfdefault}{m}{sl}
\SetMathAlphabet{\mathsfit}{bold}{\encodingdefault}{\sfdefault}{bx}{n}

\usepackage{hyperref}
\usepackage{url}
\usepackage{amsmath,amssymb,amsfonts}
\usepackage{amsthm}
\usepackage{mathtools}
\usepackage{bm}
\usepackage{bbm}
\usepackage{enumitem}
\usepackage{graphicx}
\usepackage{subcaption}

\usepackage{algorithm}
\usepackage{algpseudocode}
\usepackage{mathrsfs}

\newtheorem{theorem}{Theorem}[section]
\newtheorem{proposition}[theorem]{Proposition}
\newtheorem{lemma}[theorem]{Lemma}
\newtheorem{assumption}[theorem]{Assumption}
\newtheorem{corollary}[theorem]{Corollary}

\theoremstyle{definition}
\newtheorem{definition}[theorem]{Definition}

\theoremstyle{remark}

\title{A Function-Space Approach to the Statistical Mechanics of Learning Dynamics}

\author{Yizhou Zhang$^1$, Weichen Wu$^2$, Lun Du$^3$, Zhengjie Miao$^4$ \\
$^1$ Variational AI, $^2$ The Voleon Group, $^3$ Independent Researcher, $^4$ Simon Fraser University \\ 
\texttt{\{zyizhou96,weichen.wu.1996,dulun2834\}@gmail.com},\quad \texttt{zhengjie@sfu.ca} 
}

\def\month{MM}  % Insert correct month for camera-ready version
\def\year{YYYY} % Insert correct year for camera-ready version
\def\openreview{\url{https://openreview.net/forum?id=XXXX}} % Insert correct link to OpenReview for camera-ready version

\begin{document}
\maketitle
% ===== Abstract =====
\begin{abstract}
In the kernel regime, the preferences of neural-network learning are inherited from a frozen spectrum. In the feature-learning regime, the geometry governing learning evolves, yet trained networks still display systematic preferences, such as a bias toward simple, smooth directions. We develop a statistical framework that explains where such preferences come from. The framework is formulated directly in function space: functions and the operators governing their evolution are the primary objects, and the parameterization enters as an underlying constraint through the multiplicity of parameter configurations realizing each function. For mean-squared loss, the error relaxes exactly under the learning operator $M=JJ^\ast$, the evolving counterpart of the NTK. The stochasticity of training turns $M$ into a Gaussian weight over function-space states, while the parameter multiplicity contributes a statistical operator $B$, the curvature of the log multiplicity. A Laplace expansion of the resulting stationary distribution around the current state yields a log-determinant fluctuation free energy, $\Phi_{\mathrm{fluc}}(M;B)=\frac{\sigma_\xi^2}{2}\log\det(M^{-1}+B)+\mathrm{const}$, analogous to an Occam factor. Under mild statistical conditions, this free energy is stationary exactly when $[M,B]=0$, is minimized by pairing large eigenvalues of $M$ with small eigenvalues of $B$, and exerts a local restoring force against mismatch. The learning operator is thus biased to assign its fastest relaxation to the entropically cheapest directions. Because the bias is entropic, it grows with the noise scale of training and vanishes in the deterministic limit, which places it outside gradient-flow analyses of operator alignment. For ReLU networks, we show that the least rearrangement of activation boundaries needed to move the function along a direction is bounded by the smoothness of that direction and controls its entropic curvature, so smooth directions are learned fast, in a data-adaptive sense, while the kernel evolves.
\end{abstract}

\section{Introduction}
\label{sec:introduction}

What do neural networks learn, and in which directions do their
features move? In the neural tangent kernel regime the answer is fixed
at initialization: a highly nonlinear model reduces to an approximately
linear dynamics governed by a kernel that stays nearly frozen, and the
kernel's eigenspectrum dictates which function-space directions are
learned quickly \cite{jacot2018ntk,chizat2019lazy}. The price of this
closure is that the geometry governing learning is effectively frozen.
Practical networks operate in a feature-learning regime, in which
representations, kernels, and their spectra evolve during training
\cite{woodworth2020kernelrich}. An evolving geometry could in principle
drift in arbitrary ways. Instead, trained networks display strikingly
consistent preferences: learned functions favor smooth and simple
directions, and the feature geometry aligns with the task and with the
remaining error. This raises the question we address:
\begin{equation}
    \boxed{
    \parbox{0.8\linewidth}{\centering\itshape
    Why does the evolving geometry of feature learning carry
    systematic preferences, and what determines them?}
    }
    \label{eq:intro-central-question}
\end{equation}

A substantial literature has documented these preferences and analyzed
their consequences. Phenomenologically, networks exhibit a spectral bias
toward low-frequency components \cite{rahaman2019spectral}, and learned
representations organize into geometry-adaptive harmonic bases ordered
by smoothness relative to the data \cite{kadkhodaie2024geometry}.
Tangent features and leading NTK eigendirections align with the target
\cite{baratin2021implicit,kopitkov2019neural}, and the neural feature
ansatz ties learned weight geometry to average gradient outer products
\cite{radhakrishnan2024mechanism,beaglehole2024feature}. Theoretically,
the clearest mechanistic account of such biases comes from the kernel
regime, where they are inherited from a fixed spectrum. Beyond that
regime, fixing the dynamical geometry is no longer available, and
allowing it to adapt reintroduces nonlinear, self-consistent, and often
model-dependent dynamics. The dominant strategy is therefore to adopt a
concrete description of how features evolve---through a specific
asymptotic limit, an adaptive kernel, or a dynamical mean-field
ansatz---and to derive its consequences
\cite{lauditi2025adaptive,lauditi2026spectral,bordelon2025feature,lobo2026reshape}.
This strategy has produced sharp results, including precise accounts of
how feature learning reshapes the spectrum and improves neural scaling
laws. Its starting point is a description of feature evolution, which is
the very object whose preferences we seek to explain, and a mechanism
for why the evolving geometry favors particular configurations remains
largely open. \citet{nystrom2026commutator} makes this gap explicit.
Analyzing commutators between average gradient outer products and
neural feature matrices, together with their internal counterparts, he
shows that gradient flow alone does not drive them to zero: decay
requires an additional mechanism, such as an intrinsic damping of the
mismatch, and gradient flow does not supply one.

Our approach begins with a change in what we regard as the driving
physics. Statistical-mechanical ideas have long been used to study
neural networks through Gibbs distributions, high-dimensional landscapes,
stochastic-gradient diffusions, and collective order parameters
\cite{bahri2020statistical,mandt2016variational,mandt2017sgd,chaudhari2018sgd}.
In most formulations, however, the statistical description is built
either directly in parameter space or on summary variables chosen for a
specific model or limit. Following the view that macroscopic
organization is best formulated at its own level rather than derived
from microscopic detail \cite{anderson1972more}, we place the
statistical description in function space, where the learned object
lives, and treat the parameterization as an underlying constraint. The
two levels then play distinct roles: the learning dynamics determine
how function-space states evolve, while the parameterization determines
which states admit many parameter realizations and how those
realizations are organized. A concrete architecture thus enters as one
realization of a function-space statistical law. It
determines which dynamical operators and multiplicities are realizable,
while the form of the organizing principle is set at the function-space
level. This level is natural for another reason. Scalar observables
such as the loss retain only the magnitude of the prediction error and
discard its directional and spectral structure, whereas parameter space
retains every redundant degree of freedom. Function space lies between
these two extremes: it coarse-grains over redundant parameterizations
while preserving the geometry of learning.

For mean-squared loss this program takes a concrete form. The error
obeys the exact function-space dynamics
\begin{equation}
    \dot e=-Me,
    \qquad
    M=JJ^\ast,
    \label{eq:intro-error-dynamics}
\end{equation}
so the spectrum and eigendirections of the learning operator $M$, the
evolving counterpart of the NTK, determine how the error relaxes. The
stochasticity of SGD turns $M$ into a Gaussian statistical weight over
function-space states, while the parameterization contributes a
parameter multiplicity: the volume of parameter configurations
realizing each function. Thus $M$ describes the dynamical geometry of
learning, while the curvature $B$ of the log multiplicity describes the
statistical geometry induced by the parameterization. Expanding the
resulting stationary distribution around the current state, in the spirit of a Laplace
approximation, separates the local free energy into two pieces. The
first comes from the current error and favors aligning the fastest
relaxation direction with the residual. The second is a fluctuation
term of log-determinant form, analogous to the Occam factor in Bayesian
evidence:
\begin{equation}
    \boxed{
    \Phi_{\mathrm{fluc}}(M;B)
    =
    \frac{\sigma_\xi^2}{2}
    \log\det\!\left(M^{-1}+B\right)
    +\mathrm{const}.
    }
    \label{eq:intro-operator-free-energy}
\end{equation}
Analyzing this term yields our central result. Under mild statistical
conditions, the fluctuation free energy is stationary exactly at
commuting configurations and is minimized by reverse spectral pairing,
\begin{equation}
    \boxed{
    [M,B]=0,
    \qquad
    m_{\mathrm{large}}\longleftrightarrow b_{\mathrm{small}},
    }
    \label{eq:intro-matching}
\end{equation}
with a local restoring force against rotational mismatch. In words, the
learning operator is biased to assign its fastest
relaxation rates to the entropically cheapest, highest-multiplicity
directions. Because this bias is entropic, its strength grows with the
noise scale of training, the effective temperature of SGD, and it
vanishes in the deterministic limit. It therefore lies outside gradient-flow analyses, and it
provides, at the function-space level, a signed restoring force against
operator mismatch, of the kind that the gradient alone does not supply.

The framework states this conclusion in terms of entropic cost alone.
Its meaning in a given model follows from asking what makes a direction
entropically cheap there. A ReLU network bends only at its activation
boundaries, so displacing the function along a direction requires
rearranging those boundaries. We show that the least rearrangement
needed is bounded by how violently the derivatives of the direction
vary, and that, under mild statistical conditions, the entropic
curvature of a direction is controlled by this least rearrangement.
Smooth directions are therefore entropically cheap, and the general
principle becomes
\begin{equation}
    \boxed{
    \text{smooth}
    \;\Longrightarrow\;
    b_{\mathrm{small}}
    \;\Longrightarrow\;
    m_{\mathrm{large}}.
    }
    \label{eq:intro-final-matching}
\end{equation}
Within this class of function spaces, smooth directions of the
data-weighted functional geometry are learned fast, giving a mechanism
for the preference for simple, smooth functions that operates while the
kernel evolves.

\paragraph{Contributions.}
Our main contributions are:
\begin{itemize}
    \item We develop a statistical-mechanical formulation of
    neural-network training directly in function space. The
    statistical law is formulated before any architecture is chosen,
    and the parameterization enters through a parameter multiplicity
    over function-space states, so that the dynamical geometry $M$ and
    the statistical geometry $B$ appear as distinct objects.

    \item We derive the fluctuation free energy
    $\Phi_{\mathrm{fluc}}(M;B)$ and prove that its rotational
    stationary points are exactly the commuting configurations
    $[M,B]=0$, that its minimum pairs large eigenvalues of $M$ with
    small eigenvalues of $B$, and that it exerts a local restoring
    force reducing $\|[M,B]\|$. This yields an entropic,
    temperature-dependent source of operator alignment.

    \item We decompose the local free energy into an error-directed
    term and an entropic term, which separates residual alignment and
    multiplicity-driven alignment as two distinct preferences of the
    evolving geometry.

    \item For ReLU networks, we show that the least rearrangement of
    activation boundaries needed to move the function along a direction
    is bounded by the smoothness of that direction, and that, under mild
    statistical conditions, it controls the entropic curvature. Smooth
    directions are therefore learned fast, in a data-adaptive sense.
\end{itemize}

The remainder of the paper develops these results as a chain of
statements. Section~\ref{sec:conditional-stat-mech} constructs the
function-space stationary distribution and the local free energy.
Section~\ref{sec:operator-matching} identifies the two orientational
preferences of the learning operator, residual alignment and entropic
matching. Section~\ref{sec:relu-geometry} makes smoothness precise for ReLU
networks and shows that smooth directions are entropically cheap and
therefore learned fast. Technical derivations
and proofs are collected in the appendices.

\section{Related Work}
\label{sec:related-work}

\paragraph{Kernel limits and feature learning.}
In the infinite-width NTK limit, gradient descent is governed by an
approximately fixed kernel whose eigenspectrum sets mode-dependent
learning rates \cite{jacot2018ntk}, closely related to the lazy regime
in which the network stays near its initial linearization
\cite{chizat2019lazy}. Feature-learning analyses characterize how
kernels and spectra evolve in controlled limits
\cite{woodworth2020kernelrich,lauditi2025adaptive,lauditi2026spectral},
how a few feature-learning updates selectively amplify signal-aligned
eigenvalues and mix leading eigendirections \cite{lobo2026reshape}, and
how feature learning improves neural scaling laws
\cite{bordelon2025feature}. A related strategy enlarges the tracked
state with higher-order observables until closure is reached in a
particular model or regime \cite{NEURIPS2023_2b950a29}. These works
obtain closed dynamics for the kernel within a chosen limit, ansatz, or
enlarged state. Our learning operator $M=JJ^\ast$ is the evolving
counterpart of the NTK, and its error dynamics $\dot e=-Me$ hold
exactly without linearization (Sec.~\ref{sec:function-space-dynamics}).
The evolution of $M$ itself involves an open hierarchy of higher-order
parameter geometry; we retain $(e,M)$ as macroscopic variables and
represent the remaining higher-order drift statistically, through a
parameter multiplicity over function-space states
(Sec.~\ref{sec:parameter-microstates}).
The two approaches are complementary: closed descriptions specify how
the kernel evolves in a given regime, while the present framework
identifies a statistical preference on that evolution which is stated
independently of any particular limit.

\paragraph{Operator alignment and commutators.}
Trained networks repeatedly develop alignment among their
feature-related operators: tangent features concentrate along
task-relevant directions \cite{baratin2021implicit}, leading NTK
eigendirections move toward the target \cite{kopitkov2019neural}, and
the neural feature ansatz relates weight-Gram geometry to average
gradient outer products \cite{radhakrishnan2024mechanism}, an
alignment that has been traced to the singular structure of gradient
updates \cite{beaglehole2024feature}. \citet{nystrom2026commutator}
formulates this compatibility through layerwise commutators between
gates, backward sensitivities, covariances, AGOPs, and neural feature
matrices, and shows that their decay under gradient flow requires an
additional mechanism, such as a geometric error bound or intrinsic
damping of the mismatch. His analysis is carried out for deterministic
gradient flow and identifies stochastic training and the implicit
selection among the many internal geometries compatible with
interpolation as open directions. Our framework addresses both at the
function-space level. The parameter multiplicity formalizes the
multiplicity of parameter realizations of a given function, and under
training noise the fluctuation free energy turns this multiplicity into
a signed restoring force toward $[M,B]=0$
(Theorem~\ref{thm:operator-matching},
Corollary~\ref{cor:rotational-restoring},
Sec.~\ref{sec:commutator-interpretation}). The commutators differ in
kind---his are layerwise and parameter-side, ours act between the
function-space learning operator and the entropic curvature---so the
two analyses are complementary.

\paragraph{Bayesian statistical mechanics and kernel renormalization.}
Integrating over network parameters can produce nontrivial effective
functional geometries. For deep linear networks, backpropagating kernel
renormalization propagates finite-width effects through depth
\cite{li2021statistical}; for nonlinear Bayesian networks in the
proportional regime, partition-function methods yield renormalized
kernels described by a few order parameters
\cite{pacelli2023statistical}, extended to convolutional architectures
through local kernel renormalization \cite{aiudi2025local} and to deep
nonlinear networks through an equivalent-Wishart ansatz
\cite{baglioni2026kernel}. Our fluctuation free energy has the same
log-determinant structure that arises when Gaussian fluctuations are
integrated out, and it plays the role of an Occam factor
(Corollary~\ref{cor:local-free-energy}). Two differences shape the
results. First, the statistical weight comes from the training
dynamics: at fixed $M$, the part of the SGD dynamics closed on $(e,M)$
defines the Gaussian weight
$w_0(e\mid M)\propto\exp[-\tfrac{1}{2\sigma_\xi^2}\langle e,M^{-1}e\rangle_p]$
(Proposition~\ref{prop:dynamical-gibbs}), in place of a
Bayesian posterior. Second, the dynamical geometry $M$ and the
statistical geometry $B$ are kept as distinct objects, and it is their
relative orientation that the free energy constrains.

\paragraph{Stochastic gradient dynamics and entropy in parameter space.}
Constant-step SGD admits local diffusion and Ornstein--Uhlenbeck
approximations with effective stationary distributions
\cite{mandt2016variational,mandt2017sgd}, while realistic minibatch
noise is anisotropic and can drive genuinely nonequilibrium behavior
\cite{chaudhari2018sgd}; statistical-mechanical theories of learning
more broadly are reviewed in \cite{bahri2020statistical}. We work with
an isotropic function-space source. It is realized by online training
with label noise once the label-noise variance dominates the squared
residual, with noise scale equal to the learning rate times the
label-noise variance, and exactly by training with freshly drawn
isotropic output perturbations; finite-sample minibatch noise is
instead residual-weighted and rank deficient
(Appendix~\ref{app:temperature}). Entropic quantities
have also been studied directly in parameter space: local entropy and
the volume of wide flat regions relate to optimization and
generalization \cite{baldassi2016unreasonable,baldassi2021wide}, and
densities of states have been computed as a function of the scalar
loss \cite{mele2025density}. Our parameter multiplicity is defined over
a function-space state rather than a loss value, so its curvature
$B$ retains directional information, which is precisely what the
operator-matching results use.

\paragraph{Spectral bias and smoothness.}
Neural networks learn simpler, lower-frequency components of a target
earlier in training \cite{rahaman2019spectral}, and learned
representations in diffusion denoisers organize into geometry-adaptive
harmonic bases ordered by smoothness relative to the data
\cite{kadkhodaie2024geometry}. Together these results suggest that
simplicity is best measured relative to a data-adaptive geometry rather
than a fixed Fourier basis. In the kernel regime, the corresponding
bias is inherited from a fixed spectrum. In the present framework,
simplicity is measured by the structural cost of realizing a direction
in the represented function class, and the bias toward it arises from
the entropic matching principle, which remains in force while the
kernel evolves (Sec.~\ref{sec:smooth-learning-bias}).

\paragraph{Piecewise-linear geometry of ReLU networks.}
ReLU networks represent continuous piecewise-affine functions on a
partition of input space into activation regions, as developed through
spline descriptions of deep networks \cite{balestriero2018spline}.
Variational and representer theorems connect such networks to
spline-like function spaces with second-order regularity
\cite{unser2019representer}. Within each region the function is affine,
and its second-order structure lives on activation boundaries. We use a
coarse-grained version of this boundary structure as the structural
coordinate of the ReLU function class
(Sec.~\ref{sec:relu-linear-structural-field}), and large-deviation
arguments over its local contributions identify the eigenvalues of $B$
with the structural cost of each direction
(Theorem~\ref{thm:structural-entropic-curvature}).

\section{Function-Space Statistical Mechanics of Learning}
\label{sec:conditional-stat-mech}

A statistical description of learning requires two ingredients: an
energy, which specifies which function-space states the dynamics
favors, and an entropy, which counts how many parameter configurations
realize each state. The first is supplied by the training dynamics and
the second by the parameterization. This section constructs both in
function space and combines them into a single stationary distribution
over function-space states. The parameter-level information that the
function-space state does not capture is retained statistically,
through a parameter multiplicity.

\subsection{The learning operator and the closure problem}
\label{sec:function-space-dynamics}

Let $p(x)$ denote the data distribution and define
\begin{equation}
    \mathcal H=L^2(\mathcal X,p),
    \qquad
    \langle f,g\rangle_p
    =
    \int_{\mathcal X}f(x)g(x)p(x)\,dx.
\end{equation}
For a model $f_\theta$ and target function $y$, define
\begin{equation}
    e_\theta=f_\theta-y,
    \qquad
    \mathcal L(\theta)=\frac12\|e_\theta\|_p^2,
\end{equation}
and let $J_\theta=D_\theta f_\theta$ be the Jacobian from parameter
perturbations to function-space perturbations, so that
$\nabla_\theta\mathcal L=J_\theta^\ast e_\theta$. The deterministic
gradient-flow component $d\theta_{\rm gf}=-J_\theta^\ast e_\theta\,dt$
then induces, by the chain rule, an exact evolution of the error.

\begin{definition}[Learning operator]
\label{def:learning-operator}
The \emph{learning operator} of the model $f_\theta$ is
\begin{equation}
    \boxed{
    M_\theta\equiv J_\theta J_\theta^\ast ,
    }
    \label{eq:learning-operator}
\end{equation}
so that the gradient-flow component of the error obeys
\begin{equation}
    \boxed{
    de_{\rm gf}=-M_\theta e_\theta\,dt .
    }
    \label{eq:error-gradient-flow}
\end{equation}
For a finite-parameter model, $M_\theta$ has finite rank. Its
\emph{active sector} is
\begin{equation}
    \mathcal H_a
    \equiv
    \overline{\operatorname{Ran}M_\theta},
    \qquad
    n=\dim\mathcal H_a<\infty,
    \label{eq:active-sector}
\end{equation}
with orthogonal projector $P_a$. The restriction of $M_\theta$ to
$\mathcal H_a$ is positive definite, and all inverses and determinants
below are understood on this sector.
\end{definition}

Equation~\ref{eq:error-gradient-flow} is an exact function-space
identity and does not require a fixed-network linearization. The
spectrum and eigendirections of $M_\theta$ determine how the current
error relaxes: along an eigendirection with eigenvalue $m$, the error
decays at rate $m$. In the kernel or lazy-training regime one obtains a
closed description by further taking $M_{\theta_t}\simeq M_{\theta_0}$,
so that the geometry governing error relaxation is effectively frozen
\cite{jacot2018ntk,chizat2019lazy}. More generally, $M_{\theta_t}$ is
itself a dynamical object: it is the evolving counterpart of the neural
tangent kernel, and its spectrum and eigendirections change along the
training trajectory. Moving from parameter space to function space
therefore identifies $M_\theta$ as the instantaneous operator that
organizes how the current error relaxes, while leaving its evolution
open.

That evolution is where the difficulty lies. Differentiating
$M_\theta=J_\theta J_\theta^\ast$ along a parameter trajectory gives
schematically
\begin{equation}
    \dot M_\theta
    =
    \dot J_\theta J_\theta^\ast
    +
    J_\theta\dot J_\theta^\ast,
    \qquad
    \dot J_\theta
    =
    D_\theta^2 f_\theta[\dot\theta].
    \label{eq:closure-hierarchy}
\end{equation}
The evolution of the first-order tangent geometry depends on
second-order parameter geometry, and evolving the latter introduces
still higher derivatives. The function-space state therefore fails to
close on itself:

\begin{center}
\fbox{\parbox{0.93\linewidth}{
\textbf{The closure problem.}
The function-space state $(e,M)$ evolves under dynamics that depend on
progressively higher-order parameter geometry, and distinct parameter
configurations realizing the same $(e,M)$ can drive different
evolutions of $M$. How can the evolving geometry of learning be
characterized without resolving this hierarchy?
}}
\end{center}

Existing approaches handle the hierarchy by choosing a model or a limit
in which it becomes tractable. The kernel regime truncates it at the
first level by freezing $M$. Feature-learning theories select an
asymptotic limit or an ansatz in which the hierarchy truncates or
self-averages, and derive closed dynamics for the kernel within it
\cite{woodworth2020kernelrich,lauditi2025adaptive,lauditi2026spectral,bordelon2025feature}.
A third strategy enlarges the tracked state with additional
higher-order observables until closure is achieved in a particular
model or regime \cite{NEURIPS2023_2b950a29}. In each case, the
higher-order terms are controlled by the choice of model or limit. We
take a different route: we retain $(e,M)$ as the macroscopic variables,
keep the part of the dynamics that closes on them explicitly, and
represent the remaining higher-order information statistically.

\subsection{The It\^{o} correction and the base dynamics}
\label{sec:bare-dynamics}

We now add the stochastic component of SGD in the diffusion
approximation. We focus on the regime in which the function-space
stochastic source is approximately isotropic on the active sector,
which admits the following stochastic differential equation in a parameter-space representation
\begin{equation}
    d\theta
    =
    -J_\theta^\ast e_\theta\,dt
    +
    \sqrt{2}\,\sigma_\xi J_\theta^\ast dW_t ,
    \label{eq:parameter-space-sgd-sde}
\end{equation}
where $W_t$ is an isotropic Wiener process on the retained
function-space sector. This representation entails no loss of
generality, since any parameter-space source is represented by a
temperature operator on $\mathcal H_a$; the isotropic regime is
realized, for example, by online training with label noise, in which
case the noise scale $\sigma_\xi^2$, which plays the role of an effective
temperature, equals the learning rate times the label-noise variance (Appendix~\ref{app:temperature}).

Because $f_\theta$ is nonlinear in the parameters, pushing
Eq.~\ref{eq:parameter-space-sgd-sde} to function space requires
It\^{o}'s formula for stochastic analysis rather than the deterministic chain rule alone. The quadratic variation of the stochastic increment generates a correction
that involves the parameter Hessian of the model output,
\begin{equation}
    \Gamma_\theta
    \equiv
    D_\theta^2 f_\theta ,
    \qquad
    \left[\Gamma_\theta:A\right](x)
    =
    \sum_{ij}
    \frac{\partial^2 f_\theta(x)}
         {\partial\theta_i\partial\theta_j}
    A_{ij},
\end{equation}
and the full function-space pushforward reads
(Appendix~\ref{app:pushforward})
\begin{equation}
\boxed{
    de
    =
    \Big[
        \underbrace{-Me}_{\text{closed on }(e,M)}
        +
        \underbrace{\sigma_\xi^2\,
        \Gamma_\theta:(J_\theta^\ast J_\theta)}_{\text{It\^{o} correction}}
    \Big]dt
    +
    \sqrt{2}\,\sigma_\xi M\,dW_t .
}
\label{eq:full-function-space-sde}
\end{equation}

Equation~\ref{eq:full-function-space-sde} is the closure problem in
its stochastic form. The instantaneous pair $(e,M)$ determines both the
deterministic relaxation term $-Me$ and the covariance of the projected
stochastic increment. The It\^{o} correction, however, depends on
$\Gamma_\theta$ and on $J_\theta^\ast J_\theta$, neither of which is
determined by $(e,M=J_\theta J_\theta^\ast)$ alone. Two parameter
configurations may realize the same instantaneous function-space state,
\begin{equation}
    e_{\theta_1}=e_{\theta_2},
    \qquad
    J_{\theta_1}J_{\theta_1}^\ast
    =
    J_{\theta_2}J_{\theta_2}^\ast=M,
\end{equation}
while nevertheless satisfying
\begin{equation}
    \Gamma_{\theta_1}:
    (J_{\theta_1}^\ast J_{\theta_1})
    \neq
    \Gamma_{\theta_2}:
    (J_{\theta_2}^\ast J_{\theta_2}).
\end{equation}
The map $\theta\mapsto(e,M)$ is therefore many-to-one: it identifies
parameter configurations whose full stochastic pushforwards differ, and the higher-order geometry enters the function-space
dynamics precisely through the It\^{o} correction.

This structure, together with the difficulty of controlling
higher-order terms without committing to a model or limit, motivates
the separation at the core of our construction. The terms fixed by
$(e,M)$ define the \emph{base dynamics}
\begin{equation}
\boxed{
    de_{\rm bare}
    =
    -Me\,dt
    +
    \sqrt{2}\,\sigma_\xi M\,dW_t ,
}
\label{eq:bare-conditional-sde}
\end{equation}
whose drift and stochastic covariance are completely specified by the
instantaneous function-space state. We treat the base dynamics as the
source of an energy and the It\^{o} correction as a source of
statistical structure. The base dynamics is a preconditioned Langevin
dynamics, with preconditioner $M^2$ and noise scale $\sigma_\xi^2$, and
it has a stationary distribution of its own.

\begin{proposition}[Dynamical weight]
\label{prop:dynamical-gibbs}
At fixed $M$, the base dynamics Eq.~\ref{eq:bare-conditional-sde} has the
stationary density on $\mathcal H_a$
\begin{equation}
\boxed{
    w_0(e\mid M)
    =
    \frac1{Z_0}
    \exp\left[
        -\frac{U_M(e)}{\sigma_\xi^2}
    \right],
    \qquad
    U_M(e)
    =
    \frac12
    \langle e,M^{-1}e\rangle_p ,
}
\label{eq:bare-gibbs-density}
\end{equation}
and it is the unique stationary density toward which the base dynamics
relaxes.
\end{proposition}

The training dynamics thus supplies a Gaussian weight whose energy is
the squared norm of the error in the geometry of $M$. Directions in
which $M$ is large relax quickly and carry a small energy cost for a
given displacement; directions in which $M$ is small relax slowly and
carry a large one. The proof is given in
Appendix~\ref{app:proof-dynamical-gibbs}.

\subsection{Parameter multiplicity and entropic curvature}
\label{sec:parameter-microstates}

The It\^{o} correction carries the information that distinguishes
parameter configurations realizing the same function-space state. We
now represent it statistically, by averaging over the parameter
configurations compatible with a given function-space state
\cite{ellis1985entropy}.

To work with finite-dimensional densities, introduce a retained
function-space sector
$\mathcal H_a\subseteq\mathcal H_R\subseteq\mathcal H$ with
$N_R=\dim\mathcal H_R<\infty$ and projector $P_R$, and write
$e_R=P_R(f_\theta-y)$ for the retained error. Averaging the It\^{o} correction over the parameter configurations
compatible with $(e_R,M)$ leads to a Markov diffusion for
$e_R$ that reproduces its one-time marginals exactly, by the Markovian
projection theorem \cite{gyongy1986mimicking,brunick2013mimicking};
since the diffusion coefficient in
Eq.~\ref{eq:full-function-space-sde} depends only on $(e,M)$, only the
higher-order drift is averaged (Appendix~\ref{app:averaged-dynamics}).
Let $p_\ast(e_R\mid M)$ denote the stationary density of this
averaged dynamics at fixed $M$, which exists under mild regularity
conditions on the averaged higher-order drift
(Assumption~\ref{ass:regular-drift}).

\begin{definition}[Parameter multiplicity]
\label{def:density-of-states}
The \emph{parameter multiplicity} is the part of the stationary law not
accounted for by the dynamical weight,
\begin{equation}
\boxed{
    \Omega(e_R\mid M)
    \equiv
    \frac{p_\ast(e_R\mid M)}{w_0(P_ae_R\mid M)} .
}
\label{eq:density-of-states}
\end{equation}
The multiplicity is defined relative to a statistical background of
training: the optimizer and its hyperparameters, the noise scale, and
the data distribution. Throughout the paper this background is held
fixed during training and suppressed from the notation.
\end{definition}

The weight $w_0$ carries the component of the dynamics closed on
$(e,M)$, while $\Omega$ carries all statistical structure induced by
the averaged It\^{o} correction. Intuitively, $\Omega(e_R\mid M)$
measures how many parameter configurations realize the function-space
state $e_R$, in the sense of a parameter volume weighted by how often
training visits each configuration; in statistical mechanics it is the
density of states. In the reversible case, $\Omega$ is exactly the
potential whose gradient, preconditioned by $M^2$, reproduces the
averaged higher-order drift (Appendix~\ref{app:averaged-dynamics}). The constructions below are instantaneous, at fixed $M$.

\begin{definition}[Entropic curvature]
\label{def:entropic-curvature}
At a current state $r\in\mathcal H_R$, the \emph{entropic
curvature} is
\begin{equation}
\boxed{
    B(r)
    \equiv
    -\sigma_\xi^2
    \nabla_{e_R}^2\log\Omega(e_R)
    \big|_{e_R=r},
}
\label{eq:define-B-operator}
\end{equation}
and its compression to the active sector is
\begin{equation}
\boxed{
    B_a(r)
    \equiv
    P_a B(r)P_a\big|_{\mathcal H_a}.
}
\label{eq:compress-B-operator}
\end{equation}
\end{definition}

For a unit direction $\phi$, the quadratic form
$\langle\phi,B(r)\phi\rangle_p$ measures how quickly the number of
parameter realizations decreases as the function moves away from $r$
along $\phi$. A direction with large entropic curvature is
\emph{entropically expensive}: few parameter configurations realize
displacements along it. A direction with small entropic curvature is
\emph{entropically cheap}. Thus $M$ describes the dynamical geometry of
learning, while $B$ describes the statistical geometry induced by the
parameterization. The compression Eq.~\ref{eq:compress-B-operator}
corresponds to conditioning on the components of $e_R$ in
$\mathcal H_R\cap\mathcal H_a^\perp$ at their values in $r$.

\subsection{The conditional stationary distribution and the local free energy}
\label{sec:local-microstate-curvature}

Combining the dynamical weight with the parameter multiplicity,
Definition~\ref{def:density-of-states} can be rewritten in Gibbs form,
\begin{equation}
    p_\ast(e_R\mid M)
    =
    \frac1Z
    \exp\left[
        -\frac{F_M(e_R;z)}{\sigma_\xi^2}
    \right],
    \qquad
    F_M(e_R;z)
    =
    U_M(P_ae_R)
    -
    \sigma_\xi^2
    \log\Omega(e_R\mid M),
    \label{eq:effective-gibbs-density}
\end{equation}
which defines an effective free energy $F_M$ over function-space
states. Its two terms have distinct origins: the first is fixed by the
component of the SGD dynamics closed on $(e,M)$, and the second is the
log multiplicity of parameter configurations distinguished only by
information beyond $(e,M)$. We now examine this distribution near the
current retained state $r$, writing
\begin{equation}
    e_R=r+\delta e,
    \qquad
    \delta e\in\mathcal H_a ,
    \qquad
    r_a=P_ar .
    \label{eq:decompose-conditional-center}
\end{equation}
The point $r$ is the prescribed center of the local expansion; it is
not required to minimize the free energy.

\begin{theorem}[Local fluctuation law]
\label{thm:conditional-ensemble}
If $H_a(r)\equiv M^{-1}+B_a(r)\succ0$ on $\mathcal H_a$, then to
quadratic order around the current state $r$ the local conditional law
of $\delta e$ is Gaussian,
\begin{equation}
\boxed{
    p_{\mathrm{loc}}(\delta e\mid r,M)
    =
    \frac1{Z_{\mathrm{loc}}}
    \exp\left[
        -\frac1{2\sigma_\xi^2}
        \left\langle
            \delta e,
            H_a(r)\,\delta e
        \right\rangle_p
    \right],
}
\label{eq:local-gaussian-equilibrium}
\end{equation}
with covariance
\begin{equation}
\boxed{
    C_\ast(r,M)
    =
    \sigma_\xi^2
    \left[M^{-1}+B_a(r)\right]^{-1}.
}
\label{eq:local-stationary-covariance}
\end{equation}
\end{theorem}

The two terms in $F_M$ have distinct origins. The first is fixed by the
component of the SGD dynamics closed on $(e,M)$; the second is the log
multiplicity of parameter configurations distinguished only by
information beyond $(e,M)$. Correspondingly, local fluctuations are
constrained by two sources of stiffness: the dynamical stiffness
$M^{-1}$, which penalizes displacements along slowly relaxing
directions, and the entropic stiffness $B_a$, which penalizes
displacements that few parameter configurations can realize.
Theorem~\ref{thm:conditional-ensemble} is a conditional statement and
requires no separation between the relaxation time of $\delta e$ and
the evolution time of $r$ or $M$. The proof, including the treatment of
the linear term in the local expansion, is given in
Appendix~\ref{app:proof-conditional-ensemble}.

Integrating over the local fluctuations separates the cost of the
current state from the cost of fluctuating around it.

\begin{corollary}[Local free-energy decomposition]
\label{cor:local-free-energy}
Under the conditions of Theorem~\ref{thm:conditional-ensemble}, to
quadratic order the local free energy at the current state $r$
decomposes as
\begin{equation}
\boxed{
    \mathcal G_{\mathrm{loc}}(r,M;z)
    =
    \underbrace{
    \frac12\langle r_a,M^{-1}r_a\rangle_p
    }_{\text{error term}}
    \;-\;
    \sigma_\xi^2\log\Omega(r)
    \;+\;
    \underbrace{
    \frac{\sigma_\xi^2}{2}
    \log\det\!\left(M^{-1}+B_a(r)\right)
    }_{\Phi_{\mathrm{fluc}}(M;\,B_a(r))}
    \;+\;
    \mathrm{const},
}
\label{eq:local-total-free-energy-decomposition}
\end{equation}
where the \emph{fluctuation free energy}
$\Phi_{\mathrm{fluc}}=-\sigma_\xi^2\log Z_{\mathrm{loc}}$ is the
contribution of the local Gaussian fluctuations, and the constant
depends only on $n$ and $\sigma_\xi^2$.
\end{corollary}

The decomposition mirrors the structure of the Gaussian-process
marginal likelihood \cite{rasmussen2006gaussian}, in which a data-fit
term is balanced against a log-determinant complexity term, the Occam
factor. In the limit $B\to0$, the fluctuation free energy reduces to
$-\frac{\sigma_\xi^2}{2}\log\det M$, the familiar complexity term, and
$B$ supplies the correction contributed by parameter multiplicity. Two
of the three terms depend on how $M$ is oriented relative to the
current state; the entropy value $\sigma_\xi^2\log\Omega(r)$ does
not. The error term is the squared norm of the active residual in the
geometry of $M$. Quadratic forms of this type are well known: they
control convergence and generalization in NTK analyses
\cite{arora2019fine}, and their orientational content, which favors
residuals concentrated on the fastest relaxation directions, is the
familiar kernel--target alignment
\cite{cristianini2002kernel,cortes2012algorithms,baratin2021implicit,kopitkov2019neural}.
The preference it supplies has thus been identified and analyzed in
prior work, and Sec.~\ref{sec:residual-alignment} records how it
appears in the present framework. The new ingredient is the fluctuation
free energy, which depends on $M$ and $B$ jointly, and the remainder of
the paper focuses on the orientational preference it supplies.

\section{Two Orientational Preferences of the Learning Operator}
\label{sec:operator-matching}

Corollary~\ref{cor:local-free-energy} identifies two terms of the local
free energy that depend on how $M$ is oriented relative to the current
state. This section asks which orientations each of them favors.
Throughout, the current state $r$ and the compressed
statistical geometry $B_a(r)$ are held fixed, and $M$ is compared along
orthogonal rotations within the fixed active sector $\mathcal H_a$,
\begin{equation}
    M\;\longmapsto\;OMO^\ast,
    \qquad
    O\in SO(\mathcal H_a),
    \label{eq:fixed-spectrum-orbit}
\end{equation}
so that its eigenvalues, rank, and active subspace stay fixed and only
its eigendirections move. The comparison is conditional: it
characterizes the orientational preference supplied by each term of
the local free energy. When an actual parameter-space motion that
rotates $M$ also changes $r$ or $B_a(r)$, those responses add further
contributions to the slow dynamics of $M$; how the terms combine along
a training trajectory is discussed in Sec.~\ref{sec:discussion}.

\subsection{Residual alignment}
\label{sec:residual-alignment}

As anticipated in Sec.~\ref{sec:local-microstate-curvature}, the error
term of the local free energy depends on how $M$ is oriented relative
to the current residual.

\begin{proposition}[Residual alignment]
\label{prop:residual-alignment}
Let $r_a\neq0$. Along the orbit~\eqref{eq:fixed-spectrum-orbit}, the
error term $\frac12\langle r_a,M^{-1}r_a\rangle_p$ is minimized exactly
when $r_a$ lies in the leading eigenspace of $M$.
\end{proposition}

The error term thus favors a learning operator that relaxes the current
residual as fast as possible. This is the familiar kernel--target
alignment preference
\cite{cristianini2002kernel,cortes2012algorithms,arora2019fine}, which
the framework recovers as the dynamical part of the local free energy.
Since this preference is well understood, the remainder of this section
focuses on the fluctuation free energy, which depends on $M$ and $B$
jointly. The proof is given in Appendix~\ref{app:proofs-sec4}.

\subsection{Entropic operator matching}
\label{sec:entropic-matching}

Since $B\equiv B_a(r)$ is the Hessian of the log multiplicity, it is
symmetric, and we take $B\succeq0$.\footnote{This states that the log
multiplicity is locally concave; Sec.~\ref{sec:relu-geometry} shows that
it holds for ReLU function spaces under mild statistical conditions.}

Using $M^{-1}+B=M^{-1}(I+MB)$, the fluctuation free energy splits as
\begin{equation}
    \Phi_{\mathrm{fluc}}(M;B)
    =
    \frac{\sigma_\xi^2}{2}
    \left[
        -\log\det M
        +
        \log\det(I+MB)
    \right]
    +\mathrm{const},
    \label{eq:operator-free-energy-decomposed}
\end{equation}
where all determinants are on $\mathcal H_a$. Along the
orbit~\eqref{eq:fixed-spectrum-orbit}, $\det M$ is constant, so the
entire orientational dependence is carried by
\begin{equation}
    \boxed{
    \Phi_{\mathrm{orient}}(M;B)
    =
    \frac{\sigma_\xi^2}{2}
    \log\det(I+MB).
    }
    \label{eq:orientation-free-energy}
\end{equation}
Which orientations of $M$ does this term favor? The answer depends only
on the two spectra and on their relative orientation.

\begin{theorem}[Entropic operator matching]
\label{thm:operator-matching}
Let $B\succeq0$. For any spectrum of $M$, the orientational fluctuation
free energy $\Phi_{\mathrm{orient}}$ is rotationally stationary along
the orbit~\eqref{eq:fixed-spectrum-orbit} if and only if $[M,B]=0$, and
among these stationary configurations its global minimum is the one in
which the spectra are oppositely ordered:
\begin{equation}
    \boxed{
    [M,B]=0,
    \qquad
    m_{\mathrm{large}}
    \longleftrightarrow
    b_{\mathrm{small}}.
    }
    \label{eq:reverse-pairing}
\end{equation}
\end{theorem}

The theorem describes the orientation favored by the fluctuation
contribution: directions of weaker entropic curvature lower this
contribution when paired with larger dynamical eigenvalues. Since the
eigenvalue $m$ of $M$ is the relaxation rate of the error along the
corresponding eigendirection, and a small eigenvalue $b$ of $B$ marks an
entropically cheap direction, the theorem expresses a single principle:
\begin{equation}
    \boxed{
    \parbox{0.8\linewidth}{\centering\itshape
    The learning operator is biased to assign its fastest
    relaxation rates to the entropically cheapest directions.}
    }
    \label{eq:framework-principle}
\end{equation}
This statement is formulated entirely at the function-space level and
requires no specific architecture. The proof is given in
Appendix~\ref{app:proofs-sec4}. The configurations singled out by
Theorem~\ref{thm:operator-matching} motivate the following definitions.

\begin{definition}[Matched configuration and mismatch]
\label{def:matched}
A pair $(M,B)$ is \emph{matched} if $[M,B]=0$ and the spectra of $M$
and $B$ are oppositely ordered in a common eigenbasis. The
\emph{mismatch} of a pair is
\begin{equation}
    \boxed{
    \mathcal N(M,B)
    \equiv
    \frac12\|[M,B]\|_F^2 ,
    }
    \label{eq:commutator-mismatch}
\end{equation}
which vanishes exactly at commuting configurations.
\end{definition}

\subsection{An entropic drive toward commutation}
\label{sec:commutator-interpretation}

Trained networks repeatedly develop alignment among their
feature-related operators
\cite{baratin2021implicit,kopitkov2019neural,radhakrishnan2024mechanism,beaglehole2024feature},
and vanishing commutators are the natural measure of such
compatibility. \citet{nystrom2026commutator} shows that, under
deterministic gradient flow, the relevant commutators stabilize but
need not decay: decay requires an additional mechanism, such as an
intrinsic damping of the mismatch, and gradient flow does not supply
one. Within the present framework, the question becomes whether the
local free energy supplies such a drive. The following statement shows
that the fluctuation free energy does, and that the drive is entropic.

\begin{corollary}[Entropic restoring force]
\label{cor:rotational-restoring}
Near a matched configuration with nondegenerate spectra, parametrize
the rotations of $M$ by the angles $\theta=(\theta_{ij})_{i<j}$ in the
planes spanned by pairs of common eigendirections. The entropic force
$F_{ij}\equiv-\partial\Phi_{\mathrm{orient}}/\partial\theta_{ij}$ is
restoring,
\begin{equation}
    \boxed{
    F_{ij}=-\kappa_{ij}\,\theta_{ij}+O(|\theta|^2),
    \qquad
    \kappa_{ij}
    =
    \sigma_\xi^2\,
    \frac{(m_i-m_j)(b_j-b_i)}{(1+m_ib_i)(1+m_jb_j)}
    >0,
    }
    \label{eq:rotational-curvature}
\end{equation}
and it reduces the mismatch,
\begin{equation}
    \boxed{
    \big\langle F,\ \nabla_\theta\mathcal N\big\rangle<0
    \quad\text{for all sufficiently small }\theta\neq0 .
    }
    \label{eq:free-energy-reduces-mismatch}
\end{equation}
At fixed parameter multiplicity $\Omega$, $\kappa_{ij}\to0$ as
$\sigma_\xi\to0$.
\end{corollary}

The fluctuation free energy therefore supplies a signed restoring force
toward commuting configurations, and its descent reduces operator
noncommutativity. The force has an entropic origin: it arises from the
multiplicity of parameter configurations, weighted by the
stochasticity of training, and it vanishes in the deterministic limit.
This is why it lies outside deterministic gradient-flow analyses, and
why it can supply the kind of damping that the gradient alone does not
provide. The operators involved differ from those of
\citet{nystrom2026commutator}: his commutators are layerwise and
parameter-side, whereas ours act between the function-space learning
operator and the entropic curvature, so the two analyses are
complementary. Whether a given parameterization can realize the
corresponding operator rotations depends on the map
$\theta\mapsto M(\theta)$. The proof of
Corollary~\ref{cor:rotational-restoring} is given in
Appendix~\ref{app:proofs-sec4}.

\section{Entropic Cost in ReLU Function Spaces}
\label{sec:relu-geometry}

Theorem~\ref{thm:operator-matching} is stated entirely in terms of the
spectrum of the entropic curvature $B$: fast relaxation is assigned to
entropically cheap directions. Which directions are entropically cheap
depends on the function class realized by the parameterization. This
section answers the question for ReLU networks. The basic intuition is
that moving toward a rougher function requires generating finer
activation-boundary structure, and that each such configuration
carries an entropic cost. We derive this identification from two
structural properties of ReLU functions, linearity and local additivity
of a coarse-grained second-order field, together with mild regularity
conditions on the statistics of its local contributions.

\subsection{Smoothness and the boundary distribution}
\label{sec:relu-linear-structural-field}
 
That neural networks preferentially learn smooth and simple components
of a target is a widely shared empirical conclusion. What counts as
smooth or simple, however, varies from one account to another. Spectral
bias is measured by Fourier frequency \cite{rahaman2019spectral};
geometry-adaptive harmonic representations are ordered by smoothness
relative to the data distribution \cite{kadkhodaie2024geometry}; and in
many empirical studies, simplicity serves as an intuitive descriptor
without a formal definition. We treat simplicity in the sense in which
these studies use it, as smoothness, and make it precise for the
function class at hand.
 
A ReLU network computes a continuous piecewise-linear function: it is
affine inside each activation region and bends only where some neuron
switches on or off. Its input Hessian therefore vanishes inside the
regions and concentrates on the boundaries between them,
\begin{equation}
\boxed{
    D^2c
    =
    \sum_{F}
    \alpha_F\,
    n_Fn_F^{\top}\,
    \delta_F ,
}
\label{eq:relu-distributional-hessian}
\end{equation}
where $\delta_F$ is a surface delta function on the boundary facet $F$,
$n_F$ is its unit normal, and $\alpha_F$ is the size of the bend across
it. Where the boundaries lie, how they are oriented, and how sharply
they bend thus determine the shape of the function. To obtain a finite
description, we coarse-grain this information: partition the support of
the data distribution $p$ into $N$ blocks of equal probability mass, and
summarize the bending inside each block by a few low-order moments
(Appendix~\ref{app:structural-construction}).
 
\begin{definition}[Boundary distribution]
\label{def:structural-field}
The \emph{boundary distribution} of a network function $c$ is the
coarse-grained distribution of its activation boundaries over the
input space,
\begin{equation}
\boxed{
    h=Lc,
    \qquad
    L\equiv \mathcal C D^2 ,
}
\label{eq:structural-field}
\end{equation}
where $\mathcal C$ records, block by block, how much boundary bending
the block contains and in which orientation.
\end{definition}
 
The boundary distribution has two useful properties. First, $L$ is
linear in the function. Second, $h$ is an average of local
contributions: substituting Eq.~\eqref{eq:relu-distributional-hessian}
into Eq.~\eqref{eq:structural-field} gives
\begin{equation}
    \boxed{
    h=\frac1N\sum_{i=1}^N q_i ,
    }
    \label{eq:structural-empirical-field}
\end{equation}
where $q_i$ summarizes the activation boundaries in block $i$. The
boundary distribution also determines how the function changes.
 
\begin{proposition}[Linear readout]
\label{prop:linear-readout}
If the blocks are fine enough
(see Appendix \ref{app:resolution_condition}), there is a linear map $G$,
determined by the blocks and the data distribution, such that every
change of the retained error, up to its affine part, is the image of
the corresponding change of the boundary distribution,
\begin{equation}
\boxed{
    \delta e=G\,\delta h .
}
\label{eq:retained-linear-image}
\end{equation}
\end{proposition}
 
Consider now a direction $\phi$ in which the function can be displaced.
These are the same function-space directions along which the error
relaxes and fluctuates in Secs.~\ref{sec:conditional-stat-mech}
and~\ref{sec:operator-matching}. By
Proposition~\ref{prop:linear-readout}, displacing the function by
$\varepsilon\phi$ requires a change $\delta h$ of the boundary
distribution with $G\,\delta h=\varepsilon\phi$. Many such changes
exist, since some rearrangements of the boundaries cancel at the
resolution of the data and leave the function unchanged. The relevant
quantity is the least rearrangement that does the job.
 
\begin{definition}[Boundary cost and boundary distance]
\label{def:structural-cost}
The \emph{boundary cost} of a retained direction $\phi$ is the least
change of the boundary distribution that displaces the function by
$\phi$,
\begin{equation}
\boxed{
    \mathcal R(\phi)
    \equiv
    \min\big\{\|\delta h\|^2:\ G\,\delta h=\phi\big\}.
}
\label{eq:structural-cost}
\end{equation}
The \emph{boundary distance} between two functions is
$d(f,g)\equiv\mathcal R(f-g)^{1/2}$: the least rearrangement of
activation boundaries that turns one into the other.
\end{definition}
 
The boundary cost is a positive quadratic form on retained directions,
so displacing the function by $\varepsilon\phi$ costs
$\varepsilon^2\mathcal R(\phi)$ and $d$ is a distance. Equivalently,
$\mathcal R(\phi)$ is the squared norm of $\phi$ in the
reproducing-kernel Hilbert space of the kernel $GG^\ast$
(Appendix~\ref{app:structural-construction}). The main result of this
subsection is that the boundary cost is controlled by how violently the
derivatives of $\phi$ vary.
 
\begin{proposition}[Smoothness bounds the boundary distance]
\label{prop:smoothness-bound}
For every non-affine retained direction $\phi$,
\begin{equation}
\boxed{
    d(f,f+\phi)^2
    =
    \mathcal R(\phi)
    \;\le\;
    \|L\phi\|^2 ,
}
\label{eq:smoothness-bound}
\end{equation}
with equality when the retained sector resolves the block scale. For
directions that vary slowly on the block scale, $\|L\phi\|^2$ is a
data-weighted energy of the second derivatives of $\phi$; for uniformly
distributed one-dimensional inputs,
\begin{equation}
    \|L\phi\|^2
    \;\propto\;
    \int\big(\phi''(x)\big)^2\,dx ,
    \label{eq:roughness-1d}
\end{equation}
with a constant set by the block size.
\end{proposition}
 
Smooth directions are therefore close in boundary distance: a direction
whose derivatives vary gently can be realized with a small, sparse
rearrangement of activation boundaries, whereas a direction whose slope
must change rapidly and repeatedly requires dense, sharply bent
boundaries. In one dimension the right-hand side of
Eq.~\eqref{eq:roughness-1d} is the roughness penalty of cubic smoothing
splines \cite{wahba1990spline}, and a mode of frequency $\omega$ has
$\|L\phi\|^2\propto\omega^4$; in higher dimensions the bending is
measured on blocks of equal probability mass and therefore adapts to
the data distribution. The proofs are given in
Appendix~\ref{app:structural-construction}.

\subsection{Regular structural statistics}
\label{sec:structural-regularity}

For a fixed ReLU realization, the decomposition above follows directly
from the geometry. Across parameter configurations compatible with the
same coarse state, the $q_i$ become random structural variables.
Integrating Eq.~\eqref{eq:retained-linear-image}, the retained error is
$e_R=Gh+c_0$, with $c_0$ fixed by the affine part, and it inherits the
additive structure of $h$:
\begin{equation}
    e_R-c_0
    =
    Y_N
    \equiv
    \frac1N\sum_{i=1}^N G q_i ,
    \label{eq:retained-structural-average}
\end{equation}
an average of local contributions in the finite-dimensional space
$\mathcal H_R$. The relevant statistics are those defining $\Omega_z$,
namely $p_\ast/w_0$; since $\log w_0$ is of order one while
$\log\Omega_z$ grows linearly in $N$, the leading-order conclusions below
are the same if $p_\ast$ is used instead. The question is then
statistical: under what conditions does this large collection of local
structural contributions admit a stable macroscopic entropy geometry?

We require that the local contributions be statistically regular, as is
expected when the $q_i$ are
weakly dependent \cite{bolthausen1982,dobrushinshlosman1987,ellis1985entropy}.
Concretely, consider an external field coupled linearly to the
structural field, chosen so that the current state is typical. We
require two conditions on the resulting structural free energy.

\begin{description}
\item[\textnormal{(R1)} \emph{Regular structural free energy.}]
Under perturbations of the external field, the free energy per block
remains finite and varies smoothly, and it controls the local shape of
the distribution of the structural field. Physically, the system lies
in a single thermodynamic phase: the free energy has no singularities,
so there is no phase transition, no critical point, and no domination
by a few exceptional configurations.

\item[\textnormal{(R2)} \emph{Nondegenerate structural response.}]
The second-order response of this free energy to the external field,
which equals the per-block covariance $\chi_{\rm str}$ of the structural
field, is finite and nonzero in
every structural direction. Physically, every structural direction
fluctuates, and no fluctuation diverges: the system has no frozen
directions and lies away from criticality.
\end{description}

The two conditions describe the same free energy at two levels:
smoothness at first order and nondegeneracy at second order. Their
precise form, in terms of the scaled cumulant-generating function of
$Y_N$, is given in Appendix~\ref{app:structural-regularity}.

\subsection{Rough directions are entropically expensive}
\label{sec:structural-entropy-curvature}

By Eq.~\eqref{eq:retained-structural-average}, a change of the function
is an average of many local activation-boundary contributions, and
these contributions vary across the parameter configurations that
realize the current state. Random boundary configurations readily
produce simple changes, which need little boundary structure. A rough
change requires many local contributions to deviate together, which few
configurations achieve. Fewer realizations means a larger entropic
curvature, so the entropic cost of a direction tracks its structural
cost.

\begin{theorem}[Entropic curvature and structural cost]
\label{thm:structural-entropic-curvature}
Under the resolution condition and conditions (R1)--(R2), let
$0<\lambda_-\le\lambda_+$ bound the structural covariance in every
structural direction. Then for every unit retained direction $\phi$,
\begin{equation}
    \boxed{
    \frac{\sigma_\xi^2}{\lambda_+}\,\mathcal R(\phi)-o(1)
    \;\le\;
    \frac1N\langle\phi,B\phi\rangle_p
    \;\le\;
    \frac{\sigma_\xi^2}{\lambda_-}\,\mathcal R(\phi)+o(1).
    }
    \label{eq:B-cost-bounds}
\end{equation}
\end{theorem}

Up to an overall factor $N$ and the bounded anisotropy
$\lambda_+/\lambda_-$ of the structural statistics, the entropic
curvature of a direction is its structural cost. Large eigenvalues of
$B$ therefore correspond to rough directions, which require substantial
rearrangement of the activation boundaries, and small eigenvalues to
smooth directions, which require little. Since $\mathcal R(\phi)>0$ for
every nonzero retained direction, the lower bound also gives
$B_a(r)\succ0$ for all sufficiently large $N$. This closes the logical
chain with Sec.~\ref{sec:operator-matching}: regular structural
statistics supply the local concavity of the log multiplicity assumed
in Theorem~\ref{thm:operator-matching}. The proof is given in
Appendix~\ref{app:proof-structural}.

\subsection{Smooth directions are learned fast}
\label{sec:smooth-learning-bias}
 
Combining Theorems~\ref{thm:operator-matching}
and~\ref{thm:structural-entropic-curvature} with
Proposition~\ref{prop:smoothness-bound} translates the general matching
principle into a statement about smoothness.
 
\begin{corollary}[Smooth directions are learned fast]
\label{cor:simplicity-bias}
Under the conditions of Theorem~\ref{thm:structural-entropic-curvature},
let $M$ be at a configuration minimizing the fluctuation free energy,
and let $\phi_i,\phi_j$ be unit eigendirections of $M$. If $\phi_i$ is
smoother than $\phi_j$ by more than the anisotropy of the structural
statistics,
\begin{equation}
    \|L\phi_i\|^2
    \;<\;
    \frac{\lambda_-}{\lambda_+}\,\mathcal R(\phi_j)-\varepsilon
    \qquad\text{for some }\varepsilon>0,
    \label{eq:simplicity-bias}
\end{equation}
then $m_i\ge m_j$ for all sufficiently large $N$:
\begin{equation}
\boxed{
    \text{smooth}
    \;\Longrightarrow\;
    b_{\mathrm{small}}
    \;\Longrightarrow\;
    m_{\mathrm{large}}.
}
\label{eq:smooth-fast-pairing}
\end{equation}
\end{corollary}
 
Along an eigenmode of $M$ with eigenvalue $m$, the gradient-flow error
dynamics obey $\dot e_m=-m e_m$, so larger $m$ means faster relaxation.
The fluctuation free energy therefore favors a learning operator that
learns smooth directions fast, and this preference acts on the evolving
operator $M$ rather than on a kernel fixed at initialization. When the
retained sector resolves the block scale, the condition becomes
symmetric, $\|L\phi_i\|^2<\frac{\lambda_-}{\lambda_+}\|L\phi_j\|^2-\varepsilon$;
in the one-dimensional setting of Proposition~\ref{prop:smoothness-bound},
it says that an eigenmode of lower frequency is learned faster whenever
the two frequencies differ by more than a factor
$(\lambda_+/\lambda_-)^{1/4}$, the ordering familiar from spectral bias.
The notion of smoothness involved is data-adaptive: the data
distribution enters through the $L^2(p)$ geometry of the retained
sector, through the equal-mass blocks on which boundaries are
coarse-grained, and through the readout $G$. Whether and how a given
architecture realizes this preference depends on the map
$\theta\mapsto M(\theta)$. The proof of
Corollary~\ref{cor:simplicity-bias} is given in
Appendix~\ref{app:proof-structural}.

\section{Discussion and Future Work}
\label{sec:discussion}

The function-space picture developed here is broadly consistent with
several empirical regularities reported in neural networks. Classical
observations of spectral bias indicate that smoother or lower-frequency
components are often learned earlier \cite{rahaman2019spectral}. More
recently, diffusion denoisers have been found to develop
geometry-adaptive harmonic representations: their learned input--output
Jacobians organize into data-dependent eigendirections whose ordering
is closely related to smoothness \cite{kadkhodaie2024geometry}. The
operator measured in that work is the denoiser Jacobian rather than the
learning operator $M=J_\theta J_\theta^\ast$, so the comparison is
qualitative rather than a direct test. The observed organization
nevertheless matches the geometry favored by the fluctuation free
energy,
\begin{equation}
    \text{smoothness} \rightarrow b_{\mathrm{small}} \rightarrow
    m_{\mathrm{large}}.
\end{equation}

The same function-space language accommodates qualitatively different
learning regimes. If $M$ remains effectively fixed, the description
reduces to a kernel-like regime, in which the preferences of learning
are inherited from the initial spectrum. If the slow dynamics of $M$
respond to the restoring force of
Corollary~\ref{cor:rotational-restoring}, the framework supplies a bias
toward commuting, reverse-paired configurations. Apparently sharp
behavior during training can also have more than one origin: it may reflect
competition between distinct function-space states, or it may arise from a
broad hierarchy of relaxation times while the underlying dynamics
remain continuous. The latter possibility is conceptually related to
quantized models of neural scaling, in which smooth aggregate scaling
coexists with the sudden appearance of individual capabilities
\cite{michaud2023quantization}.

Grokking provides a suggestive example of the former possibility.
Previous work has connected delayed generalization to structured
representations \cite{liu2022grokking}, escape from an early
kernel-like regime in modular addition \cite{mohamadi2024grokking}, and
first-order phase transitions between representation phases in
two-layer teacher--student models \cite{rubin2024grokking}. Our
framework highlights an endogenous route by which the relative
statistical preference of competing function-space states can change
during training. The parameter multiplicity $\Omega$ and its local curvature $B_a(r)$
depend on the current error state, which evolves along training
(Sec.~\ref{sec:parameter-microstates}). If two competing branches
coexist, their total local free energies, denoted schematically by
$\mathcal G_A(r)$ and $\mathcal G_G(r)$ in the sense of
Eq.~\eqref{eq:local-total-free-energy-decomposition}, may therefore
cross. These symbols refer to full branch free energies, including the
error term, the fluctuation term, and rank-dependent constants,
so a term-by-term global branch theory is required to evaluate them.
Schematically,
\begin{equation}
    \mathcal G_A(r_\ast)=\mathcal G_G(r_\ast),
    \qquad
    \mathcal G_A(r)-\mathcal G_G(r)
    \ \text{changes sign across }r_\ast.
\end{equation}
In this picture, the training state itself acts as an endogenous
control variable. Identifying the competing branches, their barriers,
and the transition dynamics is a natural extension of the local
conditional theory developed here.

Several limitations define further directions. First, the present
construction works with an isotropic function-space source. A general
source is represented by a temperature operator $\mathfrak T$, with
function-space diffusion $D=2\sigma_\xi^2M\mathfrak TM$
(Appendix~\ref{app:temperature}); finite-sample minibatch noise, for
example, gives a residual-weighted $\mathfrak T_{\mathrm{mb}}$. When
$\mathfrak T$ fails to commute with $M$, the dynamical weight $w_0$ is
no longer of the Gaussian form~\eqref{eq:bare-gibbs-density}, the
stationary distribution is no longer reversible, and the operator
preference may differ from the one derived here. Relatedly, because the
entropic bias scales with $\sigma_\xi^2$ and vanishes in the
deterministic limit, comparing alignment under deterministic and noisy
training at matched risk provides a direct empirical test of the
mechanism; the absence of universal commutator decay under full-batch
gradient descent reported by \citet{nystrom2026commutator} is
consistent with this prediction. Second, $\Phi_{\mathrm{fluc}}(M;B)$ is
one contribution to the local free energy. The error term
$\frac12\langle r_a,M^{-1}r_a\rangle_p$ supplies the distinct
residual-alignment preference (Proposition~\ref{prop:residual-alignment}),
and further slow terms may also be present. On residual-preserving
rotations the error term is
exactly constant (Appendix~\ref{app:residual-preserving}); on the full
orbit, reading the fluctuation gradient as a component of the training
dynamics requires that the remaining contributions leave it visible.
Third, the conditional stationary distribution of
Theorem~\ref{thm:conditional-ensemble} requires no time-scale
separation, while reading it as the distribution actually sampled along
a training trajectory requires that local fluctuations settle faster
than $M$ and $r$ change.
Fourth, all derivations assume a statistical background that is fixed
during training: the multiplicity $\Omega$ is defined relative to the
optimizer, its hyperparameters, and the data distribution, and none of
these changes with time. Many practical optimizers introduce a
time-dependent background. Learning-rate schedules change the noise
scale over training; momentum buffers and the moment estimates of
adaptive methods such as Adam carry optimizer state that is not
determined by $(e,M)$ and evolves with the parameters, and adaptive
methods also reshape the noise covariance discussed above. Writing
$z_t$ for the statistical background at time $t$, the multiplicity
becomes $\Omega_{z_t}$, the averaged It\^{o} correction must be
conditioned on $z_t$ as well as on $(e_R,M)$, and the entropic
curvature $B_{z_t}$ evolves on the time scale of the background. This
raises two questions: how the matched configurations of
Theorem~\ref{thm:operator-matching} track a moving $B_{z_t}$, and
whether changes in $z_t$ can reorder the spectrum of $B$ and thereby
redirect the simplicity bias. A changing background also provides an
exogenous route by which the relative preference of competing
function-space states can shift, complementing the endogenous route
discussed above.
Finally, the operator
matching results are formulated on a fixed finite-dimensional active
sector, and whether a given parameterization can realize the
corresponding operator rotations depends on the map
$\theta\mapsto M(\theta)$.

% ===== Conclusion =====
\section{Conclusion}
\label{sec:conclusion}

We developed a statistical description of neural-network learning
directly in function space. Functions and the operators governing their
evolution are the primary objects, while parameter configurations enter
as realizations whose multiplicity constrains which functions are
statistically accessible. This separation distinguishes the dynamical
geometry of learning, the learning operator $M$, from the statistical
geometry induced by the parameterization, the entropic curvature $B$.

For mean-squared loss, the part of the SGD dynamics closed on the
function-space state supplies a Gaussian dynamical weight, and the
parameter multiplicity supplies the remaining statistical structure.
A local expansion of the resulting stationary distribution separates an
error term, which recovers the familiar residual-alignment preference,
from a fluctuation free energy
\begin{equation}
    \Phi_{\mathrm{fluc}}(M;B)
    =
    \frac{\sigma_\xi^2}{2}
    \log\det(M^{-1}+B)
    +\mathrm{const}.
\end{equation}
At fixed spectrum, this term is stationary exactly when $[M,B]=0$, is
minimized by the reverse pairing
$m_{\mathrm{large}}\leftrightarrow b_{\mathrm{small}}$, and supplies a
local restoring force that reduces operator mismatch. The learning
operator is thereby biased to assign its fastest relaxation to the
entropically cheapest directions, through a mechanism that is entropic
and vanishes in the deterministic limit.

For ReLU networks, the least rearrangement of activation boundaries
needed to move the function along a direction is bounded by how
violently its derivatives vary, and under mild statistical conditions
it controls the entropic curvature, up to the bounded anisotropy of the
boundary statistics. The general principle then becomes a data-adaptive
smoothness bias: smooth directions, which require only sparse and gentle
boundary rearrangements where the data lie, are learned fast.

These results suggest that function space provides a natural level for
a statistical theory of learning. In this organization, a concrete
neural network is one realization of the theory rather than its
starting point: the organizing principle is formulated in function
space, while architecture and parameterization determine which
operators and multiplicities realize it. The same viewpoint may provide
a route to other stable regularities of learning, including
representation geometry, spectral organization, and stability, together
with the more general slow dynamics and nonequilibrium extensions
discussed in Sec.~\ref{sec:discussion}.

\bibliography{main}
\bibliographystyle{tmlr}

\appendix
\section{The Temperature Operator: Representation and Realization}
\label{app:temperature}

This appendix supports two statements about the stochastic source in
Eq.~\eqref{eq:parameter-space-sgd-sde}: that it represents an arbitrary
parameter-space stochastic source rather than restricting it, and that
the isotropic regime $\mathfrak T\approx I$ used in the main text admits
a concrete algorithmic realization.

\subsection{Representation}

Let the parameter-space stochastic increment have an arbitrary
covariance,
\begin{equation}
    \mathbb E\!\left[
    d\theta_{\mathrm{noise}}\,d\theta_{\mathrm{noise}}^\ast
    \right]
    =
    \Sigma_\theta\,dt,
    \qquad
    \Sigma_\theta\succeq 0 ,
    \label{eq:general-param-noise}
\end{equation}
with $J$ held fixed as in the conditional construction.

\begin{proposition}[Temperature representation]
\label{prop:temperature-representation}
Let $D$ denote the function-space diffusion induced by
Eq.~\eqref{eq:general-param-noise}. Then:
\begin{enumerate}
\item[(i)] $D=J\Sigma_\theta J^\ast$ and
$\operatorname{Ran}D\subseteq\operatorname{Ran}J=\mathcal H_a$;
\item[(ii)] there is exactly one $\mathfrak T\succeq0$ on $\mathcal H_a$
with
\begin{equation}
    D=2\sigma_\xi^2\,M\,\mathfrak T\,M ,
    \qquad
    \mathfrak T
    =
    \frac{1}{2\sigma_\xi^{2}}\,M^{-1}DM^{-1}\Big|_{\mathcal H_a} ;
    \label{eq:T-definition}
\end{equation}
\item[(iii)] conversely, every $\mathfrak T\succeq0$ is realized by
$\Sigma_\theta=2\sigma_\xi^2 J^\ast\mathfrak T J$, i.e.\ by
Eq.~\eqref{eq:parameter-space-sgd-sde}.
\end{enumerate}
\end{proposition}

\begin{proof}
(i) The chain rule gives $de_{\mathrm{noise}}=J\,d\theta_{\mathrm{noise}}$,
hence $D=J\Sigma_\theta J^\ast$, whose range lies in
$\operatorname{Ran}J$. (ii) On $\mathcal H_a$ the restriction of
$M=JJ^\ast$ is positive definite, so $M^{-1}$ exists there and
Eq.~\eqref{eq:T-definition} defines a unique positive semidefinite
operator; uniqueness follows since $X\mapsto MXM$ is injective on
$\mathcal H_a$. (iii) Substituting $\Sigma_\theta=2\sigma_\xi^2
J^\ast\mathfrak T J$ into (i) gives
$D=2\sigma_\xi^2 JJ^\ast\mathfrak T JJ^\ast=2\sigma_\xi^2M\mathfrak TM$.
\end{proof}

Two consequences are worth recording.

\paragraph{No loss of generality.}
By Proposition~\ref{prop:temperature-representation}, writing the noise
as in Eq.~\eqref{eq:parameter-space-sgd-sde} constrains nothing: the
modelling content of the stochastic source is carried entirely by the
isotropic regime $\mathfrak T\approx I$.

\paragraph{Gauge redundancy of $\Sigma_\theta$.}
The function-space dynamics sees $\Sigma_\theta$ only through
$J\Sigma_\theta J^\ast$. Any two parameter-space covariances differing by
a term supported on $\ker J$ induce the same $D$ and hence the same
$\mathfrak T$. The representation
$\Sigma_\theta=2\sigma_\xi^2J^\ast\mathfrak T J$ selects one
representative of this equivalence class; it removes a gauge redundancy
rather than discarding information. In particular, $\mathfrak T$
parameterizes exactly those parameter-space sources supported in
$\operatorname{Ran}J^\ast=(\ker J)^\perp$, which is where gradient-based
stochastic sources necessarily live: a per-sample parameter gradient
$\nabla_\theta f_\theta(x)$ is orthogonal to $\ker J$.

\subsection{Realization of the isotropic regime and its failure}

\paragraph{Online training with isotropic label noise.}
Consider online training with learning rate $\eta_{\mathrm{lr}}$, in
which each step draws a fresh $x\sim p$ and observes $y(x)+\eta(x)$, with
$\eta$ zero-mean, of variance $\sigma_y^2$, and independent of $x$. The
single-step stochastic gradient is
\begin{equation}
    g
    =
    \bigl(e(x)-\eta(x)\bigr)\nabla_\theta f_\theta(x) ,
    \qquad
    \mathbb E[g]=J^\ast e ,
    \label{eq:online-gradient}
\end{equation}
so the estimator is unbiased for the deterministic drift. Its covariance
is $J^\ast \mathcal T_{\mathrm{on}} J$, where $\mathcal T_{\mathrm{on}}$
is defined by
\begin{equation}
    \langle u,\mathcal T_{\mathrm{on}}v\rangle_p
    =
    \mathbb E\!\left[\bigl(e(x)-\eta\bigr)^2u(x)v(x)\right]
    -\langle e,u\rangle_p\langle e,v\rangle_p ,
\end{equation}
that is,
\begin{equation}
    \mathcal T_{\mathrm{on}}
    =
    \sigma_y^2\, I
    \;+\;
    \mathcal E
    \;-\;
    \Pi_e ,
    \label{eq:T-online}
\end{equation}
where $(\mathcal E u)(x)=e(x)^2u(x)$ is multiplication by the squared
residual and $\Pi_e=e\,e^\ast$ is the rank-one operator along the
current error field. The label contribution is exactly isotropic because
$\eta$ is independent of the sampled point; the remaining two terms are
the sampling contribution, and both are $O(\|e\|^2)$.

In the diffusion approximation with time $t=k\eta_{\mathrm{lr}}$, the
parameter-space covariance is
$\Sigma_\theta=\eta_{\mathrm{lr}}J^\ast\mathcal T_{\mathrm{on}}J$.
With the normalization $2\sigma_\xi^2=\eta_{\mathrm{lr}}\sigma_y^2$,
\begin{equation}
    \bigl\|\mathfrak T-I\bigr\|
    \;\le\;
    \frac{\|e\|_\infty^{2}+\|e\|_p^{2}}{\sigma_y^{2}} ,
    \label{eq:T1-error-bound}
\end{equation}
so the isotropic regime holds in the limit
$\sigma_y^2\gg\|e\|_\infty^2$ and is approached as training proceeds:
the label contribution does not decay, while the sampling contribution
does. In this realization the temperature $\sigma_\xi^2$ acquires a
concrete meaning as the learning rate times the label-noise variance,
rather than remaining a free parameter.

This realization is not unique. Full-batch gradient flow with a freshly
drawn isotropic output perturbation at each step yields
$\mathfrak T=I$ as well, and any two sources differing on $\ker J$ are
indistinguishable at the function-space level.

\paragraph{Failure of the isotropic regime: finite-sample minibatch sampling.}
For a fixed training set of $N$ points with Jacobian matrix $J$, residual
vector $e\in\mathbb R^{N}$, and batch size $B$ sampled with replacement,
the same computation gives a per-step gradient covariance
$J^\ast\mathcal T_{\mathrm{mb}}J$, so that
$\Sigma_\theta=\eta_{\mathrm{lr}}J^\ast\mathcal T_{\mathrm{mb}}J$, with
\begin{equation}
    \mathcal T_{\mathrm{mb}}
    =
    \frac{1}{B}
    \left[
    \frac{1}{N}\operatorname{diag}(e_i^2)
    -
    \frac{1}{N^{2}}\,e\,e^{\top}
    \right] .
    \label{eq:T-minibatch}
\end{equation}
This operator is residual-weighted rather than isotropic. When all
$e_i\neq0$, it is rank deficient: $\operatorname{rank}\mathcal T_{\mathrm{mb}}=N-1$,
with $\ker\mathcal T_{\mathrm{mb}}=\operatorname{span}\{(1/e_i)_i\}$, the
deficiency arising from subtraction of the mean gradient, which the
deterministic drift already carries.

Rank deficiency of $\mathfrak T$ does not by itself degrade the
fluctuation geometry. The bare conditional covariance $C_0$ solving
$MC_0+C_0M=2\sigma_\xi^2M\mathfrak TM$ is positive definite on
$\mathcal H_a$ precisely when $\ker\mathfrak T$ meets no eigenspace of
$M$ nontrivially, since spectral separation of $M$ mixes the source
across directions. Degeneracy therefore requires both a deficient source
and a degenerate spectrum: for $N=2$ with $J=I_2$, where $M$ is
proportional to the identity and $\mathcal T_{\mathrm{mb}}$ has rank
one, $C_0$ is genuinely rank deficient, and the fluctuation sector must
be taken as $\operatorname{Ran}C_0\subsetneq\mathcal H_a$, with the
dimension-dependent constant in Eq.~\eqref{eq:operator-free-energy}
restored accordingly.

\section{Function-Space Pushforward and the Conditional Stationary Distribution}
\label{app:conditional-ensemble}

This appendix collects the derivations behind
Sec.~\ref{sec:conditional-stat-mech}.

\subsection{The It\^{o} pushforward}
\label{app:pushforward}

To make the origin of the second-order term in
Eq.~\ref{eq:full-function-space-sde} explicit, consider a general
parameter diffusion
\begin{equation}
    d\theta
    =
    a(\theta)\,dt+G(\theta)\,dW_t .
\end{equation}
For any twice differentiable observable $g(\theta)$, the
multidimensional It\^{o} formula gives, componentwise,
\begin{equation}
    dg
    =
    D_\theta g\,d\theta
    +
    \frac12
    D_\theta^2 g:
    \bigl(GG^\ast\bigr)\,dt .
\end{equation}
The second term is the It\^{o} correction generated by the quadratic
variation of the stochastic increment. In the absence of diffusion,
$G=0$, this term vanishes and the formula reduces to the ordinary
chain rule used in Eq.~\ref{eq:error-gradient-flow}.

For Eq.~\ref{eq:parameter-space-sgd-sde},
\begin{equation}
    a=-J_\theta^\ast e_\theta,
    \qquad
    G=\sqrt{2}\,\sigma_\xi J_\theta^\ast,
    \qquad
    GG^\ast
    =
    2\sigma_\xi^2J_\theta^\ast J_\theta .
\end{equation}
Applying It\^{o}'s formula to $e_\theta=f_\theta-y$, with
$\Gamma_\theta=D_\theta^2f_\theta$, therefore gives
\begin{align}
    de
    &=
    J_\theta\,d\theta
    +
    \frac12
    \Gamma_\theta:
    (GG^\ast)\,dt
    \nonumber\\
    &=
    \left[
        -J_\theta J_\theta^\ast e
        +
        \sigma_\xi^2
        \Gamma_\theta:
        (J_\theta^\ast J_\theta)
    \right]dt
    +
    \sqrt{2}\,\sigma_\xi
    J_\theta J_\theta^\ast dW_t ,
\end{align}
which is Eq.~\ref{eq:full-function-space-sde} with
$M=J_\theta J_\theta^\ast$.

\subsection{Averaged dynamics and the parameter multiplicity}
\label{app:averaged-dynamics}

Write the higher-order drift as
\begin{equation}
    a_\theta \equiv \sigma_\xi^2\, \Gamma_\theta:(J_\theta^\ast J_\theta).
    \label{eq:higher-order-drift}
\end{equation}
Let $\xi$ denote the configuration being averaged over (the parameters,
together with any auxiliary optimizer variables), define the map
\begin{equation}
    \Psi_{R,z}:\xi
    \mapsto
    e_R
    \equiv
    P_R(f_{\theta(\xi)}-y)
    \in\mathcal H_R ,
    \label{eq:parameter-to-error-map}
\end{equation}
and write $\nu_z(d\xi\mid e_R,M)$ for the conditional law of this
configuration given the retained error and the instantaneous operator.

\begin{assumption}[Regular conditional drift]
\label{ass:regular-drift}
The conditional expectation
\begin{equation}
    \bar a_z(e_R)
    \equiv
    \mathbb E\!\left[P_R\, a_{\theta(\xi)}\,\middle|\, e_R, M, z\right]
    \label{eq:conditional-drift}
\end{equation}
exists and is $C^1$ in a neighbourhood of the current state.
The stationary density $p_\ast(e_R\mid M,z)$ of the averaged
dynamics~\eqref{eq:averaged-sde} exists and is absolutely continuous
with respect to Lebesgue volume $de_R$ on $\mathcal H_R$.
\end{assumption}

Since $\operatorname{Ran}M=\mathcal H_a\subseteq\mathcal H_R$, the
resolved drift satisfies $Me=MP_ae_R$ and the stochastic increment lies
in $\mathcal H_a$. The retained error is therefore described by the
averaged dynamics
\begin{equation}
    de_R
    =
    \big[-MP_ae_R+\bar a_z(e_R)\big]\,dt
    +\sqrt2\,\sigma_\xi M\,dW_t .
\label{eq:averaged-sde}
\end{equation}
This description is exact at the level of one-time marginals: by the
Markovian projection (mimicking) theorem
\cite{gyongy1986mimicking,brunick2013mimicking}, the marginal law of
$e_R$ under the full dynamics is reproduced by a Markov diffusion whose
coefficients are the conditional expectations of the original ones.
Because the diffusion coefficient in~\eqref{eq:full-function-space-sde}
depends only on $(e,M)$, the noise requires no averaging; only the
higher-order drift is replaced by its conditional expectation.

By Definition~\ref{def:density-of-states},
\begin{equation}
    p_\ast(e_R\mid M,z)
    =
    \frac1Z\,
    \Omega_z(e_R\mid M)\,
    w_0(P_ae_R\mid M).
\label{eq:statistical-factorization-general}
\end{equation}
The two factors are tied to the averaged drift through the stationary
probability current. On the active sector, substituting
\eqref{eq:statistical-factorization-general} into the Fokker--Planck
equation of~\eqref{eq:averaged-sde} shows that the current vanishes if
and only if
\begin{equation}
    P_a\bar a_z
    =
    \sigma_\xi^2 M^2\,\nabla_{\mathcal H_a}\log\Omega_z .
    \label{eq:drift-closure}
\end{equation}
In this case $\Omega_z$ is precisely the potential whose gradient, in
the mobility $M^2$, reproduces the averaged higher-order drift.

\paragraph{Mechanism for Assumption~\ref{ass:regular-drift}.}
Self-averaging provides a natural mechanism: $a_\theta$ is a sum over
many parameters, and its concentration on the set of configurations
compatible with $(e_R,M)$ makes $\bar a_z$ regular and insensitive to
the particular conditional law. It would also extend the averaged
description from one-time marginals to path statistics, which the
results of the main text do not require. For ReLU networks, $a_\theta$
is singular on activation boundaries for an individual configuration;
Assumption~\ref{ass:regular-drift} then requires that boundary locations
be statistically regular under the conditional law, so that the
conditional average is smooth. The same notion of statistical
regularity reappears in Sec.~\ref{sec:relu-geometry}.

\subsection{Proof of Proposition~\ref{prop:dynamical-gibbs}}
\label{app:proof-dynamical-gibbs}

We compare the bare process~\eqref{eq:bare-conditional-sde}, whose
stochastic covariance is
$\mathbb E[de_{\rm noise}de_{\rm noise}^\ast]=2\sigma_\xi^2M^2\,dt$,
with a constant-mobility overdamped Langevin process
\begin{equation}
    dx_t
    =
    -K\nabla U(x_t)\,dt
    +
    \sqrt{2TK}\,dW_t,
    \label{eq:standard-overdamped-langevin}
\end{equation}
and identify $K=M^2$ and $T=\sigma_\xi^2$. We therefore seek $U_M(e)$
satisfying $M^2\nabla_eU_M(e)=Me$. On $\mathcal H_a$,
$\nabla_eU_M=M^{-1}e$, and hence
$U_M(e)=\frac12\langle e,M^{-1}e\rangle_p$. The base process can thus be
written
\begin{equation}
    de
    =
    -M^2\nabla_eU_M(e)\,dt
    +
    \sqrt{2\sigma_\xi^2M^2}\,dW_t .
\label{eq:bare-langevin-representation}
\end{equation}
Its density $p(e,t\mid M)$ satisfies the forward equation
\begin{equation}
    \partial_t p
    =
    \nabla_e\cdot
    \left[
        M^2
        \left(
            p\nabla_eU_M
            +
            \sigma_\xi^2\nabla_ep
        \right)
    \right],
\label{eq:bare-fokker-planck}
\end{equation}
with probability current
$j=-M^2\left(p\nabla_eU_M+\sigma_\xi^2\nabla_ep\right)$. For
Eq.~\ref{eq:bare-gibbs-density},
$\nabla_ew_0=-\sigma_\xi^{-2}w_0\nabla_eU_M$, so the stationary current
vanishes and $w_0$ is the reversible stationary density.

For the second claim, define
\begin{equation}
    \mathcal F_0[p]
    =
    \int p(e)U_M(e)\,de
    +
    \sigma_\xi^2
    \int p(e)\log p(e)\,de .
\label{eq:bare-probability-free-energy}
\end{equation}
Its variational derivative is
$\delta\mathcal F_0/\delta p=U_M+\sigma_\xi^2\log p+\mathrm{const}$, and
integration by parts yields
\begin{equation}
    \frac{d\mathcal F_0}{dt}
    =
    -
    \int p(e)
    \left\langle
        \nabla_e
        \frac{\delta\mathcal F_0}{\delta p},
        M^2
        \nabla_e
        \frac{\delta\mathcal F_0}{\delta p}
    \right\rangle_p
    de
    \le0,
\end{equation}
with equality only at $p=w_0$. Since $w_0$ is the unique minimizer of
$\mathcal F_0$, the base dynamics relaxes to $w_0$.
\hfill$\square$

\subsection{Proof of Theorem~\ref{thm:conditional-ensemble} and Corollary~\ref{cor:local-free-energy}}
\label{app:proof-conditional-ensemble}

\paragraph{Part (i).}
Under Assumption~\ref{ass:regular-drift}, the averaged
dynamics~\eqref{eq:averaged-sde} and its stationary density exist.
Substituting Eq.~\ref{eq:bare-gibbs-density} into the
factorization~\eqref{eq:statistical-factorization-general} gives
\begin{equation}
    p_\ast(e_R\mid M,z)
    =
    \frac1Z
    \Omega_z(e_R\mid M)
    \exp\left[
        -\frac{U_M(P_ae_R)}{\sigma_\xi^2}
    \right],
\end{equation}
which is Eq.~\ref{eq:effective-gibbs-density} with
$F_M=U_M(P_ae_R)-\sigma_\xi^2\log\Omega_z$.

\paragraph{Part (ii).}
Since $U_M(P_ae_R)$ has active-sector Hessian
$\nabla_{\mathcal H_a}^2U_M=M^{-1}$, a direct expansion of the
retained-sector free energy gives
\begin{equation}
F_M(r+\delta e;z)
=
F_M(r;z)
+
\langle g_r,\delta e\rangle_p
+
\frac12
\langle\delta e,H_a(r)\delta e\rangle_p
+
o(|\delta e|^2),
\end{equation}
where
\begin{equation}
g_r
=
P_a\nabla_{e_R}F_M(e_R;z)\big|_{e_R=r},
\qquad
H_a(r)
=
M^{-1}+B_a(r).
\end{equation}
The linear term is present because the conditioned state $r$ need not
be an unconstrained stationary point of $F_M$. Because $r$ is
prescribed as the center of the local expansion, introduce the
conjugate linear source $\lambda_r=g_r$ and the excess tilted potential
\begin{equation}
\Delta\widetilde F_{M,r}(\delta e;z)
=
F_M(r+\delta e;z)
-
F_M(r;z)
-
\langle\lambda_r,\delta e\rangle_p .
\end{equation}
Its first variation vanishes at $\delta e=0$, while its Hessian is
unchanged:
\begin{equation}
\Delta\widetilde F_{M,r}(\delta e;z)
=
\frac12
\langle\delta e,H_a(r)\delta e\rangle_p
+
o(|\delta e|^2).
\label{eq:local-free-energy-expansion}
\end{equation}
The conjugate source thus fixes the prescribed local center without
altering the quadratic fluctuation geometry. Since $H_a(r)\succ0$, the
normalized Gaussian measure associated with
Eq.~\ref{eq:local-free-energy-expansion} has precision
$H_a(r)/\sigma_\xi^2$ and covariance $\sigma_\xi^2H_a(r)^{-1}$, which
gives Eqs.~\eqref{eq:local-gaussian-equilibrium}
and~\eqref{eq:local-stationary-covariance}.

\paragraph{Corollary~\ref{cor:local-free-energy}.}
The Gaussian partition function on the $n$-dimensional active sector is
\begin{equation}
    Z_{\mathrm{loc}}(M\mid B_a)
    =
    (2\pi\sigma_\xi^2)^{n/2}
    \det_{\mathcal H_a}\!\left(M^{-1}+B_a\right)^{-1/2},
    \label{eq:fast-partition-function}
\end{equation}
so that
\begin{equation}
    \Phi_{\mathrm{fluc}}(M;B_a)
    \equiv
    -\sigma_\xi^2\log Z_{\mathrm{loc}}(M\mid B_a)
    =
    \frac{\sigma_\xi^2}{2}
    \log\det_{\mathcal H_a}\!\left(M^{-1}+B_a\right)
    +\mathrm{const}.
    \label{eq:operator-free-energy}
\end{equation}
The local Laplace expansion of $-\sigma_\xi^2\log\int p_\ast$ around $r$
then gives
$\mathcal G_{\mathrm{loc}}=F_M(r;z)+\Phi_{\mathrm{fluc}}(M;B_a(r))
+\mathrm{const}$, and substituting $F_M(r;z)=\frac12\langle
r_a,M^{-1}r_a\rangle_p-\sigma_\xi^2\log\Omega_z(r)$ yields
Eq.~\ref{eq:local-total-free-energy-decomposition}. The omitted
constant depends on $n$ and $\sigma_\xi^2$; it is constant on the
fixed-rank, fixed-spectrum orbits studied in
Sec.~\ref{sec:operator-matching}, and must be restored when comparing
sectors of different rank or different active subspaces.
\hfill$\square$

\section{Proofs for Section~\ref{sec:operator-matching}}
\label{app:proofs-sec4}

Throughout, $B$ is symmetric on $\mathcal H_a$, and rotations act on
$M$ along the orbit~\eqref{eq:fixed-spectrum-orbit}.

\subsection{Proof of Proposition~\ref{prop:residual-alignment}}

Since $(OMO^\ast)^{-1}=OM^{-1}O^\ast$,
\begin{equation}
    \langle r_a,(OMO^\ast)^{-1}r_a\rangle_p
    =
    \langle O^\ast r_a,M^{-1}O^\ast r_a\rangle_p ,
\end{equation}
so rotating $M$ is equivalent to rotating the residual over the sphere
of radius $\|r_a\|_p$ in $\mathcal H_a$. The Rayleigh quotient of
$M^{-1}$ ranges over $[1/m_1,1/m_n]$, and its minimum $1/m_1$ is attained
exactly on the leading eigenspace of $M$. Hence the lower bound is
attained exactly when $O^\ast r_a$ lies in the leading eigenspace of
$M$, that is, when $r_a$ lies in the leading eigenspace of $OMO^\ast$.

\hfill$\square$

\subsection{Proof of Theorem~\ref{thm:operator-matching} and the full landscape}

Theorem~\ref{thm:operator-matching} is part of the following
description of the full rotational landscape.

\begin{proposition}[Full rotational landscape]
\label{prop:full-landscape}
Let $M\succ0$ and $B\succeq0$ have eigenvalues
$m_1\ge\cdots\ge m_n>0$ and $0\le b_1\le\cdots\le b_n$. Along the
orbit~\eqref{eq:fixed-spectrum-orbit}, $\Phi_{\mathrm{orient}}$ is
rotationally stationary if and only if $[M,B]=0$; its global minimum is
the oppositely ordered commuting configuration and its global maximum
the equally ordered one; and if both spectra are nondegenerate, every
other commuting configuration is a saddle.
\end{proposition}

\paragraph{Rotational stationarity.}
Consider an infinitesimal rotation
\begin{equation}
    M(t)=O(t)MO(t)^\ast,
    \qquad
    O(t)=e^{t\Xi},
    \qquad
    \Xi^\ast=-\Xi,
    \label{eq:orthogonal-orbit}
\end{equation}
so that $\frac{d}{dt}M^{-1}=[\Xi,M^{-1}]$. Let $H=M^{-1}+B$. Since
$\Phi_{\mathrm{orient}}$ and $\Phi_{\mathrm{fluc}}=\frac{\sigma_\xi^2}{2}
\log\det H+\mathrm{const}$ differ by a constant on the orbit,
\begin{equation}
    \frac{d\Phi_{\mathrm{orient}}}{dt}
    =
    \frac{\sigma_\xi^2}{2}
    \operatorname{Tr}
    \left(
        H^{-1}[\Xi,M^{-1}]
    \right)
    =
    \frac{\sigma_\xi^2}{2}
    \operatorname{Tr}
    \left(
        [M^{-1},H^{-1}]\Xi
    \right).
    \label{eq:Phi-rotational-variation}
\end{equation}
Because $M^{-1}$ and $H^{-1}$ are self-adjoint, $[M^{-1},H^{-1}]$ is
anti-self-adjoint. Stationarity for all anti-self-adjoint $\Xi$ is
therefore equivalent to $[M^{-1},H^{-1}]=0$. Since $H$ is invertible,
this is equivalent to $[M^{-1},H]=0$, hence to $[M^{-1},B]=0$, and
therefore to $[M,B]=0$. This part uses only that $B$ is symmetric and
$H\succ0$.

\paragraph{Extrema.}
For $B\succeq0$ and $M^{-1}\succ0$,
\begin{equation}
    OM^{-1}O^\ast+B\succ0
    \label{eq:orbit-wise-stability}
\end{equation}
for every orthogonal $O$, so the local Gaussian is normalizable
and $\Phi_{\mathrm{orient}}$ is continuous over the entire orbit. The
orbit is compact, so global extrema exist; each is stationary and,
by the first part, commuting.

At a commuting configuration specified by a permutation $\pi$,
\begin{equation}
    M=\operatorname{diag}(m_1,\ldots,m_n),
    \qquad
    B=\operatorname{diag}(b_{\pi(1)},\ldots,b_{\pi(n)}),
    \qquad
    \Phi_{\mathrm{orient}}
    =
    \frac{\sigma_\xi^2}{2}
    \sum_{i=1}^n
    \log(1+m_ib_{\pi(i)}).
    \label{eq:commuting-orientation-free-energy}
\end{equation}
Consider $m_i>m_j$ and $b_a>b_b$. The difference between same-order and
reversed pairings is
\begin{equation}
    (1+m_ib_a)(1+m_jb_b)
    -(1+m_ib_b)(1+m_jb_a)
    =(m_i-m_j)(b_a-b_b)>0.
    \label{eq:pairwise-exchange}
\end{equation}
Since $\log$ is increasing, replacing an in-order pair by a reversed
pair lowers $\Phi_{\mathrm{orient}}$. Repeated exchanges give the global
reverse ordering as the minimum; reversing the argument gives the
equally ordered configuration as the maximum.

\paragraph{Saddles.}
At a commuting configuration, the tangent space of the orbit is spanned
by independent pairwise generators $E_{ij}$. Because both $M$ and $B$
are diagonal there,
\begin{equation}
    \dot H_{ij}
    =
    \Xi_{ij}
    \left(m_j^{-1}-m_i^{-1}\right),
    \qquad i\neq j,
\end{equation}
so $\dot H$ has only the corresponding off-diagonal pair, and both
$\operatorname{Tr}(H^{-1}\ddot H)$ and
$\operatorname{Tr}(H^{-1}\dot H\,H^{-1}\dot H)$ contain only
$\Xi_{ij}^2$ terms, with no cross-pair contributions. The second
variation is therefore diagonal in this basis, with pairwise curvature
\begin{equation}
    \kappa_{ij}
    =
    -\sigma_\xi^2
    \frac{
        (m_i-m_j)
        (b_{\pi(i)}-b_{\pi(j)})
    }{
        (1+m_ib_{\pi(i)})
        (1+m_jb_{\pi(j)})
    }.
    \label{eq:general-pairwise-curvature}
\end{equation}
For nondegenerate spectra, any commuting configuration other than the
two extremes contains both an in-order and a reversed pair, and
therefore has tangent directions of both positive and negative
curvature. It is thus a saddle.
\hfill$\square$

\subsection{Proof of Corollary~\ref{cor:rotational-restoring}}

At a matched configuration, a rotation in the plane of the pair
$(i,j)$ acts only on the corresponding $2\times2$ block, and the
determinant $\det(I+MB)$ factorizes into this block times a constant.
Write $M_\ast=\operatorname{diag}(m_i,m_j)$ and
$B=\operatorname{diag}(b_i,b_j)$ for this block, and
$M(\theta)=R(\theta)M_\ast R(\theta)^\ast$. Then
\begin{equation}
    M(\theta)
    =
    \begin{pmatrix}
        m_i-(m_i-m_j)s^2 & (m_i-m_j)sc\\
        (m_i-m_j)sc & m_j+(m_i-m_j)s^2
    \end{pmatrix},
    \qquad
    s=\sin\theta,\ c=\cos\theta,
\end{equation}
and expanding $\det[I+M(\theta)B]$, using $\det M(\theta)=m_im_j$, gives
\begin{equation}
    \det[I+M(\theta)B]
    =
    (1+m_ib_i)(1+m_jb_j)
    -
    (m_i-m_j)(b_i-b_j)\sin^2\theta .
    \label{eq:two-mode-determinant}
\end{equation}
Hence
\begin{equation}
    \Phi_{\mathrm{orient}}(\theta)
    =
    \Phi_{\mathrm{orient}}(0)
    +
    \frac{\sigma_\xi^2}{2}
    \log\left[
        1-\frac{(m_i-m_j)(b_i-b_j)\sin^2\theta}{(1+m_ib_i)(1+m_jb_j)}
    \right],
\end{equation}
and expanding to second order in $\theta$ gives
$\Phi_{\mathrm{orient}}(\theta)=\Phi_{\mathrm{orient}}(0)
+\frac12\kappa_{ij}\theta^2+O(\theta^4)$ with $\kappa_{ij}$ as in
Eq.~\ref{eq:rotational-curvature}. At a matched configuration,
$m_i>m_j$ implies $b_i<b_j$, so $\kappa_{ij}>0$.

By the proof of Theorem~\ref{thm:operator-matching}, the second
variation at a commuting configuration is diagonal in the pairwise
generators, so for several simultaneous rotations
$\Phi_{\mathrm{orient}}(\theta)=\Phi_{\mathrm{orient}}(0)
+\frac12\sum_{i<j}\kappa_{ij}\theta_{ij}^2+O(|\theta|^3)$, which gives
the force in Eq.~\ref{eq:rotational-curvature}.

For the mismatch, the off-diagonal entry of $[M(\theta),B]$ in the plane $(i,j)$
is $M(\theta)_{ij}(b_j-b_i)=(m_i-m_j)(b_j-b_i)\sin\theta\cos\theta$, and
its transpose enters with the opposite sign. Summing over independent
pairs,
\begin{equation}
    \mathcal N(M,B)
    =
    \sum_{i<j}
    (m_i-m_j)^2(b_i-b_j)^2\theta_{ij}^2
    +O(\|\theta\|^3).
    \label{eq:commutator-local-expansion}
\end{equation}
Both $\Phi_{\mathrm{orient}}$ and $\mathcal N$ are diagonal quadratic
forms in the pairwise coordinates, with positive coefficients for every
nondegenerate pair. Their gradients are therefore positively aligned. 
Writing $D_{ij}=(m_i-m_j)^2(b_i-b_j)^2>0$, one has
$\nabla_\theta\mathcal N=2D\theta+O(|\theta|^2)$ and hence
$\langle F,\nabla_\theta\mathcal N\rangle
=-2\sum_{i<j}\kappa_{ij}D_{ij}\theta_{ij}^2+O(|\theta|^3)<0$ for
sufficiently small $\theta\neq0$, which is
Eq.~\ref{eq:free-energy-reduces-mismatch}.

For the deterministic limit, substituting $b_k=\sigma_\xi^2\hat b_k$ gives
\begin{equation}
    \kappa_{ij}
    =
    \sigma_\xi^4
    \frac{(m_i-m_j)(\hat b_j-\hat b_i)}
    {(1+\sigma_\xi^2m_i\hat b_i)(1+\sigma_\xi^2m_j\hat b_j)},
\end{equation}
which vanishes as $\sigma_\xi^2\to0$ at fixed $\hat B$.
\hfill$\square$

\section{Residual-Preserving Rotations}
\label{app:residual-preserving}

If $Or_a=r_a$, then $(OMO^\ast)^{-1}=OM^{-1}O^\ast$ gives
$\langle r_a,(OMO^\ast)^{-1}r_a\rangle_p=\langle r_a,M^{-1}r_a\rangle_p$,
so the error term is constant on residual-preserving rotations. At fixed current state $r$
and fixed entropic curvature $B_a(r)$, the
entropy value $-\sigma_\xi^2\log\Omega_z(r)$ is also independent of the
rotation, so the orientational variation of the local free energy
$\mathcal G_{\mathrm{loc}}$ along these rotations is exactly the
orientational variation of $\Phi_{\mathrm{fluc}}$. This appendix records
the corresponding stationarity and pairing statements.

Write $r\equiv r_a\neq0$, and let
\begin{equation}
    G_r
    \equiv
    \left\{
        O\in SO(\mathcal H_a): Or=r
    \right\},
\end{equation}
whose infinitesimal generators satisfy $\Xi^\ast=-\Xi$ and $\Xi r=0$.
Write $B\equiv B_a(r)$, $H=M^{-1}+B$, and let
$P_\perp\equiv I-rr^\ast/\|r\|_p^2$ denote the orthogonal projector onto
$r^\perp\cap\mathcal H_a$. For a residual-preserving infinitesimal
rotation, Eq.~\ref{eq:Phi-rotational-variation} becomes
\begin{equation}
    \delta\Phi_{\mathrm{fluc}}
    =
    \frac{\sigma_\xi^2}{2}
    \operatorname{Tr}
    \left(
        [M^{-1},H^{-1}]\Xi
    \right),
    \qquad
    \Xi r=0,
\end{equation}
so stationarity with respect to all such rotations is equivalent to
\begin{equation}
    \boxed{
    P_\perp
    [M^{-1},H^{-1}]
    P_\perp
    =
    0.
    }
    \label{eq:restricted-stationarity-general}
\end{equation}
This is the exact restricted stationarity condition, without any
additional invariant-subspace assumption.

A particularly transparent case arises when the residual direction is a
common invariant mode of both operators,
\begin{equation}
    Mr=m_r r,
    \qquad
    Br=b_r r.
    \label{eq:residual-common-mode}
\end{equation}
Then $\operatorname{span}\{r\}$ and $r^\perp$ are invariant under both
$M$ and $B$, and
\begin{equation}
    M
    =
    m_r P_r
    \oplus
    M_\perp,
    \qquad
    B
    =
    b_r P_r
    \oplus
    B_\perp,
    \qquad
    P_r=rr^\ast/\|r\|_p^2 .
\end{equation}
Residual-preserving rotations act only on the $(n-1)$-dimensional
orthogonal block, and the fluctuation free energy factorizes as
\begin{equation}
    \Phi_{\mathrm{fluc}}
    =
    \frac{\sigma_\xi^2}{2}
    \log(1+m_rb_r)
    +
    \frac{\sigma_\xi^2}{2}
    \log\det_{r^\perp}
    (I_\perp+M_\perp B_\perp)
    +
    \mathrm{const}.
\end{equation}
The first term is fixed on $G_r$. Applying
Theorem~\ref{thm:operator-matching} to the orthogonal block gives
$[M_\perp,B_\perp]=0$ at restricted stationary orientations, and the
restricted minimum pairs the eigenvalues in reverse order,
\begin{equation}
    (m_\perp)_{\mathrm{large}}
    \longleftrightarrow
    (b_\perp)_{\mathrm{small}}.
\end{equation}
Thus, whenever the current residual direction is a common invariant
mode, the reverse-pairing result is an exact statement about the total
quadratic local free energy on the residual-preserving orbit. If $r=0$,
the error term vanishes, and the fluctuation free energy carries the
complete quadratic orientational dependence on the full orbit.

\section{ReLU Structural Geometry and Proofs}
\label{app:structural-regularity-proof}

This appendix gives the construction of the structural field, the
precise form of the regularity conditions of
Sec.~\ref{sec:structural-regularity}, and the proofs of
Theorem~\ref{thm:structural-entropic-curvature} and
Corollary~\ref{cor:simplicity-bias}.

\subsection{Construction of the structural field}
\label{app:structural-construction}

Partition the support of $p$ into $N$ blocks $Q_1,\dots,Q_N$ of diameter
at most $\ell$, chosen with equal $p$-mass, and let $\mathcal C$ record
in each block the zeroth and first moments of $D^2c$,
\begin{equation}
    m_i^{(0)}=\int_{Q_i}dD^2c,
    \qquad
    m_i^{(1)}=\int_{Q_i}(x-x_i)\otimes dD^2c,
    \label{eq:block-moments}
\end{equation}
where $x_i$ is a reference point of $Q_i$. The coarse-graining serves
two purposes. It places the structural field in a finite-dimensional
Hilbert space, where fluctuations of individual facet contributions are
finite, and it aggregates many facet contributions into each block.
Substituting Eq.~\ref{eq:relu-distributional-hessian} gives
\begin{equation}
    h
    =
    \sum_{F\in\mathcal F(c)} q_F,
    \qquad
    q_F
    \equiv
    \mathcal C\!\left[
        \alpha_F n_F\otimes n_F
        \mathcal H^{d-1}\!\restriction_F
    \right],
\label{eq:local-structural-charge}
\end{equation}
and grouping facet contributions by block, with the normalization that
makes $h$ intensive, yields Eq.~\ref{eq:structural-empirical-field}.

\paragraph{Resolution condition.}
\label{app:resolution_condition}
The scale $\ell$ is chosen fine enough that $\mathcal C$ is sufficient
for the retained sector: the coarse field determines the retained error
up to affine terms,
\begin{equation}
    \ker L\cap\mathrm{Aff}^\perp\subseteq\ker P_R .
    \label{eq:resolution-condition}
\end{equation}
This holds exactly as $\ell\to0$, where $L\to D^2$ and
$\ker D^2=\mathrm{Aff}$; at finite $\ell$, the residual effect of
sub-block structure on $e_R$ is controlled by the lowest unrecorded
moment.

Let $\mathcal S\equiv\overline{\operatorname{Im}L}$ denote the space of
realizable coarse structural fields, with inner product
$\langle\cdot,\cdot\rangle_h$. The affine part $\Pi_{\rm Aff}f$ is
conditioned on as part of $z$; write $f^\circ=f-\Pi_{\rm Aff}f$, so that
$Lf^\circ=Lf$.

\begin{lemma}[Linear representation]
\label{lem:linear-representation}
Under the resolution condition~\eqref{eq:resolution-condition}, there
is a unique linear map $G:\mathcal S\to\mathcal H_R$ with
$P_Rf^\circ=G\,Lf^\circ$. Consequently
\begin{equation}
    e_R=Y_N+c_z,
    \qquad
    Y_N=\frac1N\sum_{i=1}^N\tilde G q_i,
    \qquad
    c_z=P_R(\Pi_{\rm Aff}f-y),
\end{equation}
for any linear extension $\tilde G$ of $G$, and the law of $Y_N$ does
not depend on the extension.
\end{lemma}

\begin{proof}
If $Lf_1^\circ=Lf_2^\circ$, then
$f_1^\circ-f_2^\circ\in\ker L\cap\mathrm{Aff}^\perp\subseteq\ker P_R$,
so $G$ is well defined; linearity is immediate. Since
$\sum_iq_i=Nh\in\operatorname{Im}L$, every extension agrees with $G$ on
the sum.
\end{proof}

\paragraph{Proof of Proposition~\ref{prop:linear-readout}.}
By Lemma~\ref{lem:linear-representation}, $P_Rf^\circ=G\,Lf^\circ$ for
the non-affine part of every function. Differentiating along any change
of the function with the affine part held fixed gives
$\delta e=G\,\delta h$.
\hfill$\square$
 
\paragraph{Proof of Proposition~\ref{prop:smoothness-bound}.}
For a non-affine retained direction $\phi$, $GL\phi=P_R\phi=\phi$, so
$L\phi$ is admissible in Eq.~\ref{eq:structural-cost} and
$\mathcal R(\phi)\le\|L\phi\|^2$. Every admissible $\delta h$ differs from
$L\phi$ by an element of $\ker G$. If the retained sector resolves the
block scale, $\ker G=\{0\}$, the admissible change is unique, and
equality holds. Equivalently, by Step~4 of the proof of
Theorem~\ref{thm:structural-entropic-curvature},
$\mathcal R(\phi)=\langle\phi,(GG^\ast)^{-1}\phi\rangle_p$, which is the
reproducing-kernel form quoted in the main text.
 
For the second statement, the zeroth moment of $D^2\phi$ on a block
$Q_i$ is $\int_{Q_i}D^2\phi\,dx=D^2\phi(x_i)\,|Q_i|\,(1+O(\ell))$ for
$\phi$ varying slowly on the block scale, while the first moments are
smaller by a factor $O(\ell)$. Hence
\begin{equation}
    \|L\phi\|^2
    =
    \sum_i\|D^2\phi(x_i)\|_F^2\,|Q_i|^2\,(1+O(\ell))
    =
    \int\|D^2\phi(x)\|_F^2\,v(x)\,dx\,(1+O(\ell)),
\end{equation}
where $v(x)$ is the volume of the block containing $x$. Because the
blocks carry equal probability mass, $v$ is set by the data density,
which gives the data-weighted second-derivative energy. For uniformly
distributed one-dimensional inputs, $v\equiv\ell$ and
$\|L\phi\|^2=\ell\int(\phi'')^2dx\,(1+O(\ell))$, which is
Eq.~\ref{eq:roughness-1d}. For $\phi=\sin\omega x$ normalized in
$L^2$, this gives $\|L\phi\|^2\propto\omega^4$.
\hfill$\square$

Retained directions that are themselves affine are fixed by $z$; below,
$\mathcal H_R$ denotes the range of $G$.

\subsection{Precise form of the regularity conditions}
\label{app:structural-regularity}

The external field of Sec.~\ref{sec:structural-regularity} is the tilt
$s\in\mathcal H_R$ in the exponentially tilted measure
$\propto e^{\langle s,NY_N\rangle}$, and the free energy per block is the
scaled cumulant-generating function $\Lambda^G_N(s)$. The relevant tilt
is the one that makes the current state typical.

\paragraph{(R1) Regular structural free energy.}
In a neighborhood $\mathcal U\subset\mathcal H_R$ of the relevant tilt,
\begin{equation}
    \Lambda^G_N(s)
    =
    \frac1N\log\mathbb E\,e^{\langle s,\,NY_N\rangle}
    \longrightarrow\Lambda_G
    \quad\text{locally in }C^2(\mathcal U),
    \label{eq:structural-scgf}
\end{equation}
and $Y_N$ has a density of the form
\begin{equation}
    \rho_N(y)=A_N(y)\,e^{-NI_G(y)},
    \qquad
    I_G=\Lambda_G^\ast,
    \qquad
    \frac1N\|D^2\log A_N\|\to0 .
    \label{eq:app-R1}
\end{equation}
The Legendre form is the standard G\"artner--Ellis structure
\cite{gartner1977,ellis1984,ellis1985entropy}. The first requirement
expresses smoothness of the free energy under perturbations of the
field; the second expresses that this free energy controls the local
shape of the distribution, with only subextensive prefactor corrections.

\paragraph{(R2) Nondegenerate structural response.}
The structural covariance, defined as an operator on $\mathcal S$ by
\begin{equation}
    \chi_{\rm str}(t)=\lim_{N\to\infty} N^{-1}\operatorname{Cov}_t\Big(\sum_iq_i\Big),
\end{equation}
satisfies
\begin{equation}
    0<\lambda_-I\preceq\chi_{\rm str}(t)\preceq\lambda_+I<\infty
    \label{eq:app-R2}
\end{equation}
for all tilts $t=G^\ast s$ with $s\in\mathcal U$.

The $C^2$ convergence in (R1) already implies that fluctuations of the
structural sum are extensive. Weak dependence of the $q_i$, such as
summable correlations, is the natural picture behind (R1)--(R2)
\cite{bolthausen1982,dobrushinshlosman1987}, although large deviations
generally require mixing beyond summable covariances.

\subsection{Proof of Theorem~\ref{thm:structural-entropic-curvature}}
\label{app:proof-structural}

\paragraph{Step 1: Covariance of the retained average.}
Since $\mathbb E\,e^{\langle s,NY_N\rangle}
=\mathbb E\,e^{\langle G^\ast s,\sum_iq_i\rangle}$, the retained
generating function is the contraction
$\Lambda^G_N(s)=\Lambda_N(G^\ast s)$. Differentiating under the
exponentially tilted measure,
\begin{equation}
    D_s^2\Lambda^G_N(s)
    =
    \frac1N\operatorname{Cov}_s(NY_N)
    =
    G\Big[\tfrac1N\operatorname{Cov}_t\big(\textstyle\sum_iq_i\big)\Big]G^\ast,
    \qquad t=G^\ast s .
\end{equation}
In the limit,
\begin{equation}
    \chi_G\equiv D^2\Lambda_G=G\chi_{\rm str}G^\ast .
\end{equation}
Because $G^\ast$ maps $\mathcal H_R$ into $\mathcal S$, (R2) gives
$\lambda_-GG^\ast\preceq\chi_G\preceq\lambda_+GG^\ast$, and $\chi_G\succ0$
since $G$ is onto $\mathcal H_R$.

\paragraph{Step 2: Legendre duality.}
At the optimizing tilt $s_y$ of
$I_G(y)=\sup_s\{\langle s,y\rangle-\Lambda_G(s)\}$, one has
$y=D\Lambda_G(s_y)$ and $DI_G(y)=s_y$. Differentiating the first
relation gives $dy=\chi_G(s_y)\,ds_y$, hence
\begin{equation}
    D^2I_G(y)=\chi_G(s_y)^{-1}.
\end{equation}

\paragraph{Step 3: Reading off $B$.}
By Lemma~\ref{lem:linear-representation}, $e_R$ and $Y_N$ differ by a
constant shift with unit Jacobian, so
$\Omega_z(e_R)=\rho_N(e_R-c_z)$ up to normalization. Therefore
\begin{equation}
    B
    =
    -\sigma_\xi^2D^2\log\Omega_z
    =
    \sigma_\xi^2\big[N\,D^2I_G-D^2\log A_N\big]
    =
    \sigma_\xi^2N\chi_G^{-1}+o(N).
    \label{eq:B-structural}
\end{equation}

\paragraph{Step 4: Structural cost.}
Inverting the bounds of Step 1,
\begin{equation}
    \frac1{\lambda_+}(GG^\ast)^{-1}
    \preceq\chi_G^{-1}\preceq
    \frac1{\lambda_-}(GG^\ast)^{-1}.
\end{equation}
For any $h$ with $Gh=\phi$, write $h=h_0+k$ with
$h_0=G^\ast(GG^\ast)^{-1}\phi\in(\ker G)^\perp$ and $k\in\ker G$. Then
$\|h\|_h^2=\|h_0\|_h^2+\|k\|_h^2\ge\|h_0\|_h^2
=\langle\phi,(GG^\ast)^{-1}\phi\rangle$, so
\begin{equation}
    \langle\phi,(GG^\ast)^{-1}\phi\rangle=\mathcal R(\phi),
\end{equation}
which gives Eq.~\ref{eq:B-cost-bounds}. For
$\phi\in\mathcal H_R\cap\mathrm{Aff}^\perp$, one has
$GL\phi=P_R\phi=\phi$, so $L\phi$ is admissible and
$\mathcal R(\phi)\le\|L\phi\|_h^2$, with equality if and only if
$L\phi\in(\ker G)^\perp$.

\paragraph{Step 5: Stability.}
Since $(GG^\ast)^{-1}\succ0$ on $\mathcal H_R$, the lower bound in
Eq.~\ref{eq:B-cost-bounds} grows linearly in $N$ in every retained
direction, and the $o(N)$ correction is dominated for all sufficiently
large $N$. The compression to $\mathcal H_a$ inherits the bound, which
gives $B_a(r)\succ0$.

\paragraph{Pullback form.}
If $\ker G=\{0\}$, that is, if no structure is invisible to the
retained sector, then $G$ is invertible with inverse $L$ on
$\mathcal H_R$, and $\chi_G^{-1}=L^\ast\chi_{\rm str}^{-1}L$. In this case
\begin{equation}
    B
    =
    \sigma_\xi^2 L^\ast K L ,
    \qquad
    K=N\chi_{\rm str}^{-1}+o(N),
\label{eq:B-structural-pullback}
\end{equation}
so the entropic curvature is the pullback of a positive structural
metric through the roughness operator $L$, and
$\mathcal R(\phi)=\|L\phi\|_h^2$.

Finally, for an eigenmode $B\phi_i=b_i\phi_i$ with $\|\phi_i\|=1$, the
bounds give $b_i\asymp N\mathcal R(\phi_i)$. At fixed $N$ the common
extensive factor does not affect the spectral ordering, so larger
$B$-eigenvalues correspond to structurally costlier directions.
\hfill$\square$

\subsection{Proof of Corollary~\ref{cor:simplicity-bias}}
 
By Theorem~\ref{thm:structural-entropic-curvature}, $B_a\succ0$ for all
sufficiently large $N$, so Theorem~\ref{thm:operator-matching} applies.
At a minimizing configuration, $M$ and $B_a$ commute and their spectra
are oppositely ordered in a common eigenbasis, so the unit
eigendirections $\phi_i,\phi_j$ of $M$ are also eigendirections of
$B_a$, with eigenvalues $b_k=\langle\phi_k,B\phi_k\rangle_p$. By the
upper bound of Eq.~\ref{eq:B-cost-bounds} and
Proposition~\ref{prop:smoothness-bound},
\begin{equation}
    \frac{b_i}{N}
    \le
    \frac{\sigma_\xi^2}{\lambda_-}\mathcal R(\phi_i)+o(1)
    \le
    \frac{\sigma_\xi^2}{\lambda_-}\|L\phi_i\|^2+o(1),
\end{equation}
and by the lower bound,
$b_j/N\ge(\sigma_\xi^2/\lambda_+)\mathcal R(\phi_j)-o(1)$. Under
Eq.~\ref{eq:simplicity-bias},
\begin{equation}
    \frac{b_i}{N}
    <
    \frac{\sigma_\xi^2}{\lambda_+}\mathcal R(\phi_j)
    -\frac{\sigma_\xi^2\varepsilon}{\lambda_-}+o(1)
    \le
    \frac{b_j}{N}
    -\frac{\sigma_\xi^2\varepsilon}{\lambda_-}+o(1),
\end{equation}
so $b_i<b_j$ for all sufficiently large $N$. Since the spectra are
oppositely ordered, $m_i\ge m_j$.
\hfill$\square$
\end{document}